\newcommand{\arXiv}{}
\ifdefined\arXiv
\documentclass{article}

\pdfoutput=1

\usepackage[margin=.8in]{geometry}

\usepackage{authblk}

\else
\documentclass[acmsmall, anonymous]{acmart}

\AtBeginDocument{%
  }

\setcopyright{acmlicensed}
\copyrightyear{2024}
\acmYear{2024}
\acmDOI{XXXXXXX.XXXXXXX}
\fi

\newcommand{\matrixCellWidth}{4.5cm}

\usepackage{epsfig}
\usepackage{amsfonts}
\usepackage{balance}  

\ifdefined\arXiv
\usepackage[numbers]{natbib}
\usepackage{hyperref}
\fi
\usepackage{tabularx}
\usepackage{amsmath}
\usepackage{cleveref}

\usepackage{booktabs} 
\usepackage{mathtools} 

\usepackage{changepage}
\usepackage{algpseudocode}
\usepackage{bm}
\usepackage{graphicx}
\usepackage{comment}
\usepackage{array}
\ifdefined\VLDB
\newcommand{\subparagraph}{}
\fi
\usepackage{amsthm}
\usepackage{algorithm}
\usepackage{subfig}
\usepackage{blindtext}
\usepackage{multirow}
\usepackage{yfonts}
\usepackage{xcolor}
\usepackage{url}
\usepackage{balance}
\usepackage{nopageno}
\usepackage{enumitem}
\usepackage{dsfont}
\usepackage{thmtools}
\usepackage{thm-restate}

\usepackage[tableposition=top]{caption}
\usepackage[labelfont=bf]{caption}

\newtheorem*{theorem*}{Theorem}

\newcommand*{\inlineequation}[2][]{%
	\begingroup
	\refstepcounter{equation}%
	\ifx\\#1\\%
	\else
	\label{#1}%
	\fi
	\relpenalty=10000 %
	\binoppenalty=10000 %
	\ensuremath{%
		#2%
	}%
	~\@eqnnum
	\endgroup
}
\algtext*{EndFunction}
\algtext*{EndIf}
\algtext*{EndFor}
\algtext*{EndWhile}
\algtext*{EndIf}

\newif\ifcomm
\commtrue
\ifcomm
\else
\commfalse
\fi
\ifcomm
	\newcommand{\mycomm}[3]{{\footnotesize{{\color{#2} \textbf{[#1: #3]}}}}}
	\newcommand{\CRdel}[1]{\textcolor{red}{\sout{#1}}}
\else
    \newcommand{\mycomm}[3]{}
    \newcommand{\CRdel}[1]{}
\fi

\newtheorem{theorem}{Theorem}
\newtheorem{definition}{Definition}

\newtheorem{lemma}{Lemma}

\newenvironment{mystatement}{
    \begin{adjustwidth}{3em}{3em} 
    \itshape 
}{
    \end{adjustwidth}
}

\newcommand{\alg}{\emph{SkipPredict}\xspace}
\newcommand{\algFULL}{\emph{Skip or Predict}\xspace}

\newcommand{\algtwo}{\emph{DelayPredict}\xspace}

\begin{document}

\floatstyle{plaintop}
\restylefloat{table}

\title{\alg{}: When to Invest in Predictions for Scheduling}


\ifdefined\arXiv
\author[1]{Rana Shahout}

\author[1]{Michael Mitzenmacher}

\affil[1]{Harvard University, USA}
\date{}
\fi

\ifdefined\arXiv
\maketitle
\fi

\begin{abstract}

In light of recent work on scheduling with predicted job sizes, we consider the effect of the cost of predictions in queueing systems, removing the assumption in prior research that predictions are external to the system's resources and/or cost-free.
In particular, we introduce a novel approach to utilizing predictions, \alg{}, designed to address their inherent cost. Rather than uniformly applying predictions to all jobs, we propose a tailored approach that categorizes jobs based on their prediction requirements. To achieve this, we employ one-bit ``cheap predictions'' to classify jobs as either short or long. \alg{} prioritizes predicted short jobs over long jobs, and for the latter, \alg{} applies a second round of more detailed ``expensive predictions'' to approximate Shortest Remaining Processing Time for these jobs.  
Our analysis takes into account the cost of prediction. We examine the effect of this cost for two distinct models.  In the external cost model, predictions are generated by some external method without impacting job service times but incur a cost.  In the server time cost model, predictions themselves require server processing time, and are scheduled on the same server as the jobs.
\end{abstract}

\maketitle
\section{Introduction}
\label{sec:intro}
\begin{figure}[t]
   \centering
   \includegraphics[width=0.35\linewidth]{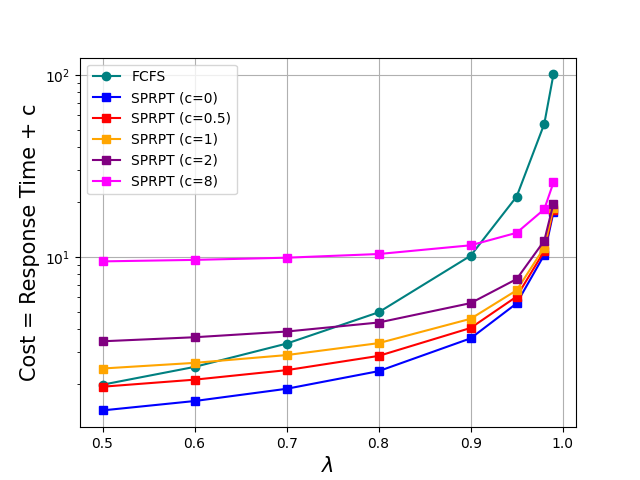}
   \caption{Considering prediction cost in SPRPT algorithm in M/M/1 system. The cost is the sum of mean response time and fixed prediction cost.}
   \label{fig:motivation_SPRPT}
\end{figure}

Machine learning research is advancing rapidly, reshaping even traditional algorithms and data structures. This intersection has led to the rise of ``algorithms with predictions'', also called learning-augmented algorithms, where classical algorithms are optimized by incorporating advice or predictions from machine learning models (or other sources). These learning-augmented algorithms have demonstrated their effectiveness across a range of areas, as shown in the collection of papers \cite{githubio} on the subject and as discussed in the surveys \cite{DBLP:books/cu/20/MitzenmacherV20,DBLP:journals/cacm/MitzenmacherV22}.  

Queueing systems are an example where the learning-augmented algorithm paradigm has been applied for scheduling. Several studies have examined queues with predicted service times rather than exact service times, generally with the goal of minimizing the average time a job spends in the system \cite{dell2015psbs,dell2019scheduling,mitzenmacher2019scheduling,mitzenmacher2021queues,wierman2008scheduling}, and additionally some recent works also consider scheduling jobs with deadlines \cite{salman2023evaluating,salman2023scheduling}.
However, existing works do not adequately model the resources required to obtain such predictions.
They often assume that predictions are provided ``for free'' when a job arrives, which may not be a realistic assumption in the practical evaluation of a system.  
Incorporating the cost of predictions is essential, as it could be argued that the resources devoted to calculating predictions might be more effectively used to directly process the jobs themselves. This perspective challenges the potential efficiency of integrating predictions into real-world queueing systems. As a result, the following questions arise:

\begin{mystatement}
When does the use of predictions, including their computation, justify their costs? Should all jobs be treated uniformly by computing predictions for each one?
\end{mystatement}

As a simple example, let us consider a model where predictions of the service time arrive with the job and do not affect the arrival or service times, but do introduce a fixed cost $c$ per job, so the total cost per job is the sum of the mean response time and fixed prediction cost. When we look at Shortest Predicted Remaining Processing Time (SPRPT) policy~\cite{mitzenmacher2019scheduling}, the improvement in the cost over FCFS naturally varies with the prediction cost, as illustrated in Figure~\ref{fig:motivation_SPRPT}. 


In this paper, we focus on the cost of predictions in settings where two stages of predictions are available.  We consider the
setting of an
an M/G/1 queueing system: jobs arrive to a single-server queue, according to a Poisson arrival process with general i.i.d service times. We examine two distinct cost models. In the first model, referred to as the external cost model, predictions are provided by an external server, thus they do not affect job service time. However, we do factor in a fixed\footnote{In Section~\ref{subsec: non_fixed_costs} we discuss the case of random costs from a distribution.} cost for these predictions. The expected overall cost per job in this model is the sum of the job's expected response time within the system and the cost associated with the time for prediction. In the second model, referred to as the server time cost model, predictions themselves require a fixed time from the same server as the jobs to produce, and hence a scheduling policy involves also scheduling the predictions. In this model, the expected overall cost per job is determined by the expected response time. As this model integrates the prediction process within the primary job processing system, it offers a different perspective on the cost implications of predictions as compared to the external model. (In particular, for heavily loaded systems, adding time for jobs to obtain predictions could lead to an overloaded, unstable system.)

As adding a single prediction to an M/G/1 model is relatively straightforward, we consider systems where we have {\em two stages} of prediction. In the first stage, we may simply predict whether a job is short or long. This type of one-bit prediction (studied in \cite{mitzenmacher2021queues}) is very natural for machine learning, and in practice may be much simpler and faster to implement. 
We therefore call these ``cheap predictions.''
In a possible second stage, we predict the service time for a job, which we refer to as ``expensive predictions,'' as in practice we expect them to be substantially more costly. (While we focus on these types of two stages, one could alternatively consider variations where the two stages could yield the same type of prediction; e.g. service time, with the cheap prediction being less accurate but consuming fewer resources.)
We introduce a scheduling policy, called \alg{} (\algFULL{}), which first categorizes jobs into short and long jobs with the first prediction, prioritizes short jobs over long ones, and restricts additional service time predictions exclusively to long jobs. \alg{} is shown in Figure~\ref{fig:skipredict_overview} and described more formally in Section~\ref{sec:policy}. We analyze the effect of the cost of prediction by considering \alg{} with the previously described external cost model and server cost model.

We compare \alg{} with three distinct previously studied policies, re-analyzing them with prediction costs in the two proposed models. First, we consider First Come First Served (FCFS), which does not require predictions (and hence incurs no cost from predictions). Second, 1bit~\cite{mitzenmacher2021queues}, a policy using only cheap predictions, separates jobs into predicted short and long jobs, and applies FCFS for each category, thereby eliminating the need for a second stage of prediction. Third, Shortest Predicted Remaining Processing Time (SPRPT) performs expensive predictions for all jobs, and no cheap predictions.
We find that service time predictions are particularly effective in high-load systems. Our analysis shows that \alg{} potentially outperforms the other policies (FCFS, 1bit, SPRPT) in both cost models, especially when there is a cost gap between cheap and expensive predictions. Additionaly, we present another alternative algorithm called \algtwo{} which avoid cheap predictions by running jobs for a fixed period before executing an expensive prediction. \algtwo{} initially schedules all jobs FCFS but limits them to a time $L$.
Jobs exceeding the limit $L$ are preempted, and given lower priority, and then they are scheduled by SPRPT.  We find \alg{} can also perform better than \algtwo{} when there is a \mbox{cost gap between these predictions.}

\subsection{Contributions and Roadmap}

This paper makes the following contributions.

\begin{itemize}
\item In Section~\ref{sec:model}, we introduce the \alg{} policy along with two cost models: the external cost model and the server cost model.

\item Section~\ref{sec:results} details our derivation of the mean response time of \alg{} for both predicted short and long jobs in each cost model, and includes a comparison of the external and server cost models.

\item In Section~\ref{sec:baselines}, we analyze 1bit and SPRPT policies in the external cost model and the server cost model and compare them with \alg{}. In Section~\ref{sec:simulation}, we empirically demonstrate the effectiveness of \alg{} over FCFS, 1bit and SPRPT.

\item In Section~\ref{sec:delayp}, we present \algtwo{} and derive its mean response time in both cost models.

\item  Section~\ref{sec:variants} discusses several possible variants of \alg{} for future study.

\end{itemize}

\begin{figure}[t]
   \centering
   \includegraphics[width=0.4\linewidth]{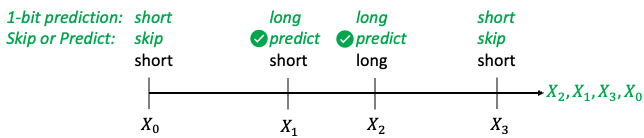}
   \caption{The \alg{} algorithm.}
   \label{fig:skipredict_overview}
\end{figure}
\section{Related Work}
\label{sec:related_works}


Our work fits in the relatively recent area of algorithms with predictions, which generally refers to the approach of aiming to improve algorithms (both in terms of theoretical bounds and in practical performance) by making use of predictions from sources such as machine-learning algorithms. 
(A brief survey of the area, including discussions of several papers, can be found in \cite{DBLP:books/cu/20/MitzenmacherV20,DBLP:journals/cacm/MitzenmacherV22}; also, see \cite{githubio} for a collection of papers in this area.) 

In scheduling, several studies have explored size-based scheduling algorithms that are informed with an estimate of the job sizes rather than the exact job size. Wierman and Nuyens~\cite{wierman2008scheduling} introduced bounds on the mean response time, mean slowdown, response-time tail, and the conditional response time of policies for a wide range of policies that schedule using inexact job-size information. The works~\cite{dell2015psbs, dell2019scheduling} primarily focus on evaluating the effects of inaccuracies on estimating job sizes. Scully and Harchol-Balter~\cite{scully2018soap} have examined scheduling policies that are based on the amount of service a job has received where the scheduler is assumed to only estimate the service received, potentially distorted by adversarial noise. Their research aims to develop scheduling policies that remain effective and robust despite these uncertainties.

In the realm of scheduling with predictions, Mitzenmacher~\cite{mitzenmacher2019scheduling} demonstrated that the analyses of various single-queue job scheduling approaches can be generalized to the context of predicted service times instead of true values. This work derives formulae for several queueing systems that schedule jobs based on predictions of service times. Also, it provides insights into how much mispredictions affect the mean response time of the analyzed policies.
In later work, Mitzenmacher~\cite{mitzenmacher2021queues} considered scheduling algorithms with a single bit of predictive advice, namely whether a job is ``large'' or ``small''; that is, if its size is above or below a certain threshold. If the advice bit is short, the job is placed at the front of the queue, otherwise it is placed at the back. The work also considers one-bit prediction schemes in systems with large numbers of queues using the power of two choices. This work shows that even small amounts of possibly inaccurate information can yield significant performance improvements.
Scully, Grosof, and Mitzenmacher~\cite{DBLP:conf/innovations/ScullyGM22} design a scheduling approach for M/G/1 queues that has mean response time within a constant factor of shortest remaining processing time (SRPT) when estimates have multiplicatively bounded error, improving qualitatively over simply using predicted remaining service times.
Azar, Leonardi, and Touitou study similar problems in the online setting, without stochastic assumptions, and consider the approximation ratio versus SRPT \cite{DBLP:conf/stoc/AzarLT21,DBLP:conf/soda/AzarLT22}.

Our work has a similar flavor to various ``2-stage'' problems, such as 2-stage stochastic programming and 2-stage stochastic optimization  (e.g., \cite{grass2016two,kolbin1977stochastic,shmoys2004stochastic,swamy2006approximation}). Here, soemwhat differently, our two stages are both predictions of service time, at different levels of specificity.  

We make extensive use of the SOAP framework (Schedule Ordered by Age-based Priority)~\cite{scully2018soap}, which was recently developed to analyze age-based scheduling policies. We use this framework to derive mean response time formulas.   
We provide a brief background to the framework in Section~\ref{sec:soap_background}.

\section{Model}
\label{sec:model}

We consider M/G/1 queueing systems with arrival rate $\lambda$. The processing times for each arriving job are independent and drawn based on the cumulative distribution $F(x)$, with an associated density function $f(x)$.

\subsection{Scheduling Algorithm: \alg{}}
\label{sec:policy}
\alg{} initially categorizes jobs based on a binary prediction of either short or long, which we refer to as a cheap prediction. 
Only jobs predicted as long are further scheduled for a detailed expensive prediction to get the predicted processing time.

With \alg{}, jobs that are predicted short have priority over all other jobs. Specifically, these short-predicted jobs are not preemptible and are scheduled based on First-Come, First-Served (FCFS), a non-sized-based policy since no predicted size is available. Jobs predicted to be long are preemptible and scheduled according to the Shortest Predicted Remaining Processing Time (SPRPT) with predicted size given by the expensive prediction.

We focus on a model where, given a threshold $T$, the cheap predictions are assumed to be independent over jobs, a job being predicted as short (less than $T$) with probability $p_T(x)$. 
Similarly, the expensive predictions are assumed to be independent over jobs, and they are given by a density function $g(x,y)$, where $g(x,y)$ is the density corresponding to a job with actual size $x$ and predicted size $y$.  Hence
$$\int_{y=0}^\infty g(x,y) dy = f(x).$$

The following quantities will be important in our analyses.  
\begin{definition}
    We define $S'_{<T}$ and $S'_{\ge T}$ to be the service times for predicted short jobs and predicted long jobs respectively. We also define the rate loads $\rho'_{< T}$ and $\rho'_{\ge T}$ respectively. 

$$\mathbb{E}[S'_{<T}] = \int_{0}^{\infty} x \cdot p_T(x) \cdot f(x) \, dx \quad 
\mathbb{E}[S'_{\ge T}] = \int_{0}^{\infty} x \cdot (1 - p_T(x)) \cdot f(x) \, dx$$

$$\mathbb{E}[{S'_{<T}}^2] = \int_{0}^{\infty} x^2 \cdot p_T(x) \cdot f(x) \, dx \quad 
\mathbb{E}[{S'_{\ge T}}^2] = \int_{0}^{\infty} x^2 \cdot (1 - p_T(x)) \cdot f(x) \, dx$$

$$\rho'_{< T} = \lambda \int_{x=0}^{\infty} xf(x)p_T(x)\, dx, \quad \rho'_{\ge T} = \lambda \int_{x=0}^{\infty} xf(x)(1-p_T(x))\, dx$$
\end{definition}

We consider two different models, external cost, in which the predictions are provided by an external server, and server cost, in which the predictions are scheduled on the same server as the job.
Next, we describe \alg{} and define the rank function for each model.

\subsection{Single-Queue, External Cost}

In this model, the predictions are provided externally, such as by an external server, with predicted long jobs going through an additional layer of prediction.  A job can be described by a triple $(x,b,r)$; we refer to this as a job's type. Here $x$ is the service time for the predictor, $b$ is the output from a binary predictor that determines whether the job is short or long, and for any long job, $r$ is the result of a service-time predictor that provides a real-number prediction of the service time.  If a job is predicted to be short, we do not consider $r$, and so we may take $r$ to be null. Again, we refer to $b$ as the cheap prediction and $r$ as the expensive prediction.

In this model, the predictions do not affect the service time of the job, and we treat the overall arrival process, still, as Poisson. Accordingly, \alg{} can be viewed as a two-class priority system: Class 1 is for short jobs, managed by FCFS within the class. Class 2 is for long jobs, according to SPRPT using service time prediction.
However, we do associate a cost with predictions.  All jobs obtain a cheap prediction at some constant fixed cost $c_1$, and all long jobs obtain an expensive prediction at some fixed cost $c_2$.
Accordingly, we can model the total expected cost for predictions per job in equilibrium as $C = c_1 + c_2 z$, where $z = \int f(x)(1-p_T(x))dx$ is the expected fraction of jobs requiring the second prediction.
In general, both $z$ and $c_1$ will depend on our choice of first layer prediction function, and similarly $c_2$ will depend on the choice of second layer prediction function. Therefore, for some parameterized families of prediction functions, we may wish to optimize our choice of predictors.
Specifically, letting $T$ be the expected response time for a job in the system in equilibrium, we might typically score a choice of predictors by the expected overall cost per job, which we model as a function $H(C,T)$, such as the sum of
the $C$ and $T$.

\begin{definition}
\label{def:external_cost}
In the external cost model, suppose $\mathbb{E}[T]_{ext}^{{PS}}$ and $\mathbb{E}[T]_{ext}^{{PL}}$ are the expected response time for predicted short job and predicted long job in the system in equilibrium respectively. Then, the total cost is $$(1-z) \cdot \mathbb{E}[T]_{ext}^{{PS}} + z \cdot \mathbb{E}[T]_{ext}^{{PL}} + C $$ where $z$ is the expected fraction of jobs requiring the second prediction and $C$ is the expected cost for prediction per job.
\end{definition}

\subsection{Single-Queue, Server Time Cost}

The server time cost model refers to the setting where predictions are scheduled on the same server as the jobs, so there is a server time cost based on a defined policy.

Jobs predicted as short are categorized as non-preemptible while in execution, thereby prioritizing their completion before predicting new jobs. However, jobs predicted as long are further scheduled for a detailed expensive prediction. Thus, cheap predictions outrank expensive predictions and long jobs. Similarly, expensive predictions are prioritized over predicted long jobs.

\alg{} now can be viewed as a four-class priority system: Class 1 is designated for short jobs, managed by FCFS within the class. Consequently, short jobs are non-preemptible, thus prioritizing them over predicting new ones. This approach is practical because even if the new jobs are predicted to be short, they are assigned behind the already running short jobs (Following FCFS). Class 2 includes jobs for cheap predictions, also operating under FCFS within their class. Class 3 involves jobs requiring expensive predictions, following the FCFS within the class. Finally, class 4 is reserved for long jobs, according to SPRPT using service time prediction from class 3.

\begin{definition}
\label{def:srv_cost}
In the server cost model, suppose $\mathbb{E}[T]_{srv}^{{PS}}$ and $\mathbb{E}[T]_{srv}^{{PL}}$ are the expected response time for predicted short job and predicted long job in the system in equilibrium respectively. Then, the total cost is $$(1-z) \cdot \mathbb{E}[T]_{srv}^{{PS}} + z \cdot \mathbb{E}[T]_{srv}^{{PL}}$$ where $z$ is the expected fraction of jobs requiring the second prediction.
\end{definition}

\section{Formulas via SOAP Analysis}
\label{sec:results}

We employ the SOAP framework~\cite{scully2018soap}, a (relatively) recently developed analysis method, to obtain precise formulas for mean response time\footnote{We note that SOAP provides for finding the Laplace-Stieltjes transform of the response time distribution;  we focus on the mean response time throughout this paper for convenience in comparisons.} of \alg{} in both the external cost model and in the server cost mode. While we could analyze the external cost model without SOAP as a two-class system, we choose to use SOAP for a consistent analysis.

\subsection{SOAP Background}
\label{sec:soap_background}
The SOAP framework can be used to analyze scheduling policies for M/G/1 queues 
that can be expressed in terms of rank functions.
Recent research by Scully and Harchol-Balter \cite{scully2018soap} has classified many scheduling policies as SOAP policies.
These policies determine job scheduling through a rank, always serving the job with the lowest rank. (In cases where multiple jobs share the lowest rank, the tie is resolved using First Come First Served.)
The rank function determines the rank of each job, and it can depend on certain static characteristics of the job, often referred to as the job's type or descriptor.
For example, the descriptor could represent 
the job's class (if the model has different classes), and other static attributes, such as its size (service time).
The rank can also depend on the amount of time the job has been served, often referred to as the job's age.    
A key assumption underlying SOAP policies is that a job's priority depends only on its own characteristics and its age, an aspect that aligns with our model and scheduling algorithm \alg{}. We refer the interested reader to \cite{scully2018soap} for more details.  

SOAP analysis uses the tagged-job technique.
That is, we consider a tagged job, denoted by $J$, of size $x_J$ and with descriptor $d_J$. 
We use $a_J$ to denote the amount of time $J$ has received service.  
The mean response time of $J$ is given by the sum of its waiting time (the time from when it enters to when it is first served) and the
residence time (time from first service to completion). To calculate the waiting time, SOAP considers the delays caused by other jobs, including ``old'' jobs that arrived before $J$ and ``new'' jobs that arrived after $J$.
A key concept in SOAP analysis is the worst future rank of a job, as ranks may change non-monotonically. The worst future rank of a job with descriptor $d_J$ and age $a_J$ is denoted by $rank_{d_J}^{\text{worst}}(a_J)$.
When $a_J=0$, the rank function is denoted by $r_{worst}= rank_{d_J}^{\text{worst}}(0)$.

In the SOAP framework, waiting time is shown to be equivalent to the transformed busy period in an M/G/1 system with arrival rate $\lambda$ and job size $X^{\text{new}}[rank_{d_J}^{\text{worst}} (a)]$ \footnote{$X^{\text{new}}[rank_{d_J}^{\text{worst}} (a)]$ represents how long a new job that just arrived to the system is served until it completes or surpasses $rank_{d_J}^{\text{worst}} (a)$}. The initial work of this period represents the delay caused by old jobs.
To deal with the delay due to old jobs, SOAP introduced a transformed system where jobs are categorized based on their rank. \emph{Discarded} old jobs, exceeding the rank threshold $r_{worst}$, are excluded from the transformed system. \emph{Original} old jobs, with a rank at or below $r_{worst}$, are considered as arrivals with rate $\lambda$ and a specific size distribution $X_{0}^{\text{old}}[r_{worst}]$ \footnote{$X_{0}^{\text{old}}[r_{worst}]$ represents how long a job is served while it is considered \emph{original} with respect to the rank $r_{worst}$}. \emph{Recycled old jobs}, currently at or below $r_{worst}$ but previously above this threshold, are treated as server vacations of length $X_{i}^{\text{old}}[r_{worst}]$\footnote{$X_{i}^{\text{old}}[r_{worst}]$ represents how long a job is served while it is considered \emph{recycled} for the $i$ time with respect to the rank $r_{worst}$} for $i \ge 1$ in the transformed system. As explained later, in \alg{} jobs could only be recycled once, so we only have $X_{1}^{\text{old}}[r_{worst}]$.

SOAP shows that, because Poisson arrivals see time averages, the stationary delay due to old jobs has the same distribution as queueing time in the transformed M/G/1/FCFS system. This system is characterized by 'sparse' server vacations, where (\emph{original}) jobs arrive at rate $\lambda$ and follow the size distribution $X_{0}^{\text{old}}[r_{worst}]$.


\begin{theorem}[Theorem 5.5 of \cite{scully2018soap}]
\label{soaptheorem}
Under any SOAP policy, the mean response time of jobs with descriptor $d$ and size $x_J$ is:
\begin{align*}
\mathbb{E}[T(x_J, d)] =  & \frac{\lambda \cdot \sum_{i=0}^{\infty} \mathbb{E}[X_i^{old}[r_{worst}]^2]}{2(1-\lambda \mathbb{E}[X_0^{old}[r_{worst}]) (1-\lambda\mathbb{E}[X^{new}[r_{worst}])} \\
& +\int_{0}^{x_J} \frac{1}{1-\lambda\mathbb{E}[X^{new}[rank_{d_J}^{\text{worst}} (a)]]} da.
\end{align*}
\end{theorem}


\subsection{Rank functions of \alg{}}
The relevant attributes to \alg{} are the size, the 1-bit prediction, and the predicted service time. 
We can model the system using descriptor $\mathcal{D} =$ [size, predicted short/long, predicted time] $= [x, b, r]$. \alg{} in the external cost model results in the following rank function:
\begin{align}
\label{eq:ext_rank_equation}
rank_{ext}([x,b, r], a) = 
\begin{cases}
\langle 1, -a \rangle & \text{if } b = 1, \\
\langle 2, r - a \rangle &  \text{ if } b = 0.
\end{cases}
\end{align}
which uses the first dimension to encode the class priority (short or long), and the second dimension to enforce the priority for each class (FCFS for short jobs, SPRPT for long jobs). In such nested rank function, the first dimension serves as the primary rank, with the priority ordering following a lexicographic ordering.

In the server cost model, \alg{} results in the following rank function:
\begin{align}
\label{srv:rank_equation}
rank_{srv}([x,b, r], a) = 
\begin{cases}
\langle 2, -a \rangle & \text{if $0 \le a \le c_1 $ (initial rank, and cheap prediction)}, \\
\langle 1, -a \rangle & \text{if } b = 1 \text{ and } a > c_1 \text{ (short jobs after cheap prediction)}, \\
\langle 3, -a \rangle & \text{if } b = 0 \text{ and } c_1 + c_2 > a > c_1 \text{ (long jobs, expensive prediction)}, \\
\langle 4, r - a \rangle &  \text{if } b = 0 \text{ and } a \geq c_1 + c_2 \text{ (long jobs after expensive prediction)}.
\end{cases}
\end{align}
Note entering jobs have ranked 2 in the first dimension, placing them behind short jobs awaiting or receiving service. Since after predictions short jobs would simply be placed
behind other short jobs, it makes sense to deprioritize the cheap predictions. On the other hand, we prioritize long predictions over long jobs to implement SPRPT.

Finally, for jobs with rank 4 in the first dimension, we use $r-a$ as the secondary rank.  Technically the predicted remaining service time is $r - (a - c_1 - c_2)$, since the job's age includes service for predictions of time $c_1 + c_2$.  As $c_1$ and $c_2$ are fixed, using $r-a$ is equivalent to using $r - (a - c_1 - c_2)$ for ranking, and we use $r-a$ for convenience.

\subsection{\alg{} in External Cost Model}

\begin{lemma}
\label{lemma_skippredict_PS_ext}
For \alg{} in the external cost model, the expected mean response time for a predicted short job, $\mathbb{E}[T]_{ext}^{{PS}}$ is
\begin{align*}
    \mathbb{E}[T]_{ext}^{{PS}} = \frac{\lambda \mathbb{E}[{S'_{< T}}^2]}{2(1-\rho'_{<T})} + \mathbb{E}[S'_{< T}].
\end{align*}

\end{lemma}

\begin{proof}
While SOAP can be used to analyze the mean response time for predicted short jobs, this case is straightforward and follows the mean response time in the system for FCFS, which is given by Equation 23.15 in~\cite{harchol2013performance}.
\end{proof}

\begin{definition}
We define $S'_{\ge T, r}$ to be the service time for jobs predicted to be long (cheap prediction says $\ge T$) where the expensive prediction of the service time is less than or equal to $r$.

$$\mathbb{E}[S'_{\ge T, r}] = \int_{y = 0}^{r} \int_{x = 0}^{\infty} (1 - p_T(x)) \cdot x \cdot g(x,y) dx dy , \quad 
\mathbb{E}[{S'_{\ge T, r}}^2] = \int_{y = 0}^{r} \int_{x = 0}^{\infty} (1 - p_T(x)) \cdot x^2 \cdot g(x,y) dx dy.$$

\end{definition}

\begin{lemma}
\label{lemma_skippredict_PL_ext}
    For \alg{} in the external cost model, if we let $a(r) =  \int_{t = r}^{\infty} \int_{x = t-r}^{\infty} (1-p_T(x)) g(x, t) (x - (t - r))^2 \, dx \, dt$, the expected mean response time for a predicted long job of true size $x_J$ and predicted size $r$ is

    \begin{align*}
    \mathbb{E}[T(x_J, r)]_{ext}^{{PL}} = &\frac{\lambda}{2(1-\rho^{ext}_r)^2} \Bigg( \mathbb{E}[{S'_{<T}}^2] + \mathbb{E}[{S'_{\ge T, r}}^2] + a(r) \Bigg) + \int_{0}^{x_J} \frac{1}{1-\rho^{ext}_{(r-a)^+}} \, da
    \end{align*}

    Where $\rho^{ext}_r= \lambda \left(\mathbb{E}[S'_{<T}] + \mathbb{E}[S'_{\ge T, r}]  \right)$ is the load due to jobs of predicted short and jobs predicted long but their service time prediction less than $r$ and $(r-a)^+ = max(r-a, 0)$.
\end{lemma}

\begin{proof}

To analyze \alg{} for a predicted long job in the external cost model using SOAP, we first find the worst future rank and then calculate $X^{\text{new}}[rank_{d_J}^{\text{worst}} (a)]$, $X_{0}^{\text{old}}[r_{worst}]$ and $X_{i}^{\text{old}}[r_{worst}]$ for predicted long job.
As described in \eqref{eq:ext_rank_equation}, the rank function for predicted long jobs is monotonic (here the job descriptor is $d_J = [x_J, 1, r]$), and every job's rank is strictly decreasing with age, thus $J$'s worst future rank is its initial rank, here: $rank_{d_J}^{\text{worst}} (a)=  \langle 2, r -a \rangle $ and $r_{worst}= rank_{d_J}^{\text{worst}} (0) = \langle 2, r \rangle$.

$X^{\text{new}}[rank_{d_J}^{\text{worst}} (a)]$:
Suppose that a new job $K$ of predicted size $r_K$ arrives when $J$ has age $a_J$.
$J$'s delay due to $K$ depends on whether $K$ is predicted to be short or long. If $K$ is predicted short then it will preempt $J$ and be scheduled till completion because it has a higher class. Otherwise, if $K$ has a predicted job size less than $J$'s predicted remaining process time ($r - a_J$), $K$ will always outrank $J$. Thus

\[X_{x_K}^{new}[\langle 2, r - a \rangle]  = 
\begin{cases}
    x_K   & \text{$K$ is predicted short } \\
    x_K \mathds{1}(r_K < r - a) & \text{$K$ is predicted long}
\end{cases}
\]

\begin{align*}
\mathbb{E}[X^{new}[\langle 2, r - a \rangle]] = & \int_{0}^{\infty} p_T(x) x f(x) dx + \int_{0}^{r-a} \int_{x = 0}^{\infty} x \cdot (1 - p_T(x)) g(x,y) dx dy \\
\end{align*}

$X_{0}^{\text{old}}[r_{worst}]$:
Whether another job $I$ is original or recycled depends on its prediction as short or long, and in the case it is long, it also depends on its predicted size relative to J's prediction.
If $I$ is predicted short, then it remains original until its completion. Otherwise, if $I$ is predicted long, $I$ is original only if $r_I \le r$, because then until its completion its rank never exceeds $r$.

\[
X_{0, x_I}^{\text{old}}[\langle 2, r \rangle] = 
\begin{cases}
    x_I   & \text{if $I$ is predicted short} \\
    x_I\mathds{1}(r_I \le r) & \text{if $I$ is predicted long}
\end{cases}
\]

\begin{align*}
\mathbb{E}[X_{0}^{\text{old}}[\langle 2, r \rangle]] = & \int_{0}^{\infty} p_T(x) x f(x) dx + \int_{y = 0}^{r} \int_{x = 0}^{\infty} (1 - p_T(x)) \cdot x \cdot g(x,y) dx dy 
\end{align*}

\begin{align*}
\mathbb{E}[(X_{0}^{\text{old}}[\langle 2, r \rangle])^2]] = &  \int_{0}^{\infty} p_T(x) x^2 f(x) dx + \int_{y = 0}^{r} \int_{x = 0}^{\infty} (1-p_T(x)) \cdot x^2\cdot g(x, y) dx dy
\end{align*}

$X_{i}^{\text{old}}[r_{worst}]$:
If another job $I$ is predicted long and if $r_I > r$, then $I$ starts discarded but becomes recycled when $r_I - a = r$. This starts at age $a = r_I -r$ and continues until completion, which will be $x_I - a_I = x_I - (r_I - r)$. Thus, for $i \ge 2$, $X_{i, x_I}^{\text{old}}[\langle 2, r \rangle] = 0$. Let $t = r_I$:

\[
X_{1, x_I}^{\text{old}}[\langle 2, r \rangle] =
\begin{cases}
    0   & \text{if $I$ is predicted short} \\
    x_I - (t - r) & \text{if $I$ is predicted long}
\end{cases}
\]

\begin{align*}
\mathbb{E}[X_{1}^{\text{old}}[\langle 2, r \rangle]^2] = \int_{t = r}^{\infty} \int_{x = t-r}^{\infty} (1-p_T(x)) \cdot g(x, t)\cdot(x - (t - r))^2 dx dt
\end{align*}

Applying Theorem~\ref{soaptheorem} leads to the result.

\end{proof}

\subsection{\alg{} in Server Time Cost Model}
As explained, considering the rank function defined in~\eqref{srv:rank_equation}, we first find the worst future rank of $J$, denoted as $rank_{d_J}^{\text{worst}}$, as follows: 
\[ rank_{d_J}^{\text{worst}} (a)= \begin{cases} \langle 2, -a \rangle & \text{if $J$ is predicted short} \\ \langle 4, r -a\rangle & \text{if $J$ is predicted long} \end{cases} \]

\begin{lemma}
For \alg{} policy in the server time cost model, the expected mean response time for a predicted short job, $\mathbb{E}[T]_{srv}^{{PS}}$ is

\begin{align*}
    \mathbb{E}[T]_{srv}^{{PS}} = \frac{\lambda \cdot \left(c_1^2 + 2c_1 \mathbb{E}[S'_{<T}] + \mathbb{E}[{S'_{<T}}^2]\right)}{2(1-\rho_{PS}^{srv})} + \mathbb{E}[S'_{<T}]
\end{align*}

where $\rho_{PS}^{srv} = \lambda \left(c_1 + \mathbb{E}[S'_{<T}]\right)$ is the load due to jobs of predicted short and their cheap prediction cost.

\end{lemma}

\begin{proof}

To analyze \alg{} for predicted short jobs, we calculate $X^{\text{new}}[rank_{d_J}^{\text{worst}} (a)]$, $X_{0}^{\text{old}}[r_{worst}]$ and $X_{i}^{\text{old}}[r_{worst}]$ for these predicted short jobs in the server cost model, where the job descriptor in this case is $(b, r) = (1,*)$.

$X^{\text{new}}[rank_{d_J}^{\text{worst}} (a)]$:
Let's consider a new job $K$ arriving when $J$ is at age $a_J$ (where $a_J \le \min(r, x_J)$). The worst rank of $J$ depends on whether $J$ is predicted to be short or long.
If $J$ is predicted short, then $J$'s worst future rank is its current rank $\langle 2, -a_J \rangle$. 
Given that $K$'s initial rank is $\langle 2, 0 \rangle$, at least equivalent to $J$'s worst future rank, the delay $J$ experiences due to $K$ is:
$X_{x_K}^{new}[\langle 2, -a_J \rangle] = 0$

$X_{0}^{\text{old}}[r_{worst}]$:
Suppose that $J$ witnesses an old job $I$ of initial size $x_I$. The duration for which $I$ remains \emph{original} depends on whether its prediction is short or long. If $I$ is predicted short, it remains original until completion. Alternatively, if $I$ is predicted long, it would remain original until the cheap prediction phase (lasting $c_1$), after which its rank shifts to $\langle 3, 0 \rangle$.

\[
X_{0, x_I}^{\text{old}}[\langle 2, 0 \rangle] = 
\begin{cases}
    c_1 + x_I   & \text{if $I$ is predicted short} \\
    c_1 & \text{if $I$ is predicted long}
\end{cases}
\]

\begin{align*}
\mathbb{E}[X_0^{\text{old}}[\langle 2, 0 \rangle]] =  c_1 + \int_{0}^{\infty} p_T(x) x f(x) dx
\end{align*}

$X_{i}^{\text{old}}[r_{worst}]$:
There are no instances of recycled jobs because either $I$ completes its service (if predicted short) or it is discarded completely (if predicted long), and thus never gets a rank lower than $\langle 2, 0 \rangle$. Thus, for $i \ge 1$, $$X_{i, x_I}^{\text{old}}[\langle 2, 0 \rangle] = 0$$

Applying Theorem~\ref{soaptheorem} yields the result.

\end{proof}

\begin{definition}
We define $S''_{<T}(c_1)$ to be the service times including prediction (cost of prediction as parameter $c_1$) for predicted short jobs.

$$\mathbb{E}[S''_{<T}(c_1)] = \int_{0}^{\infty} (x+ c_1) \cdot p_T(x) \cdot f(x) \, dx \quad 
\mathbb{E}[{S''_{<T}}^2(c_1)] = \int_{0}^{\infty} (x + c_1)^2 \cdot p_T(x) \cdot f(x) \, dx $$ 

\end{definition}

\begin{definition}
\label{def:service_time_w_price}
We define $S''_{\ge T, r}(c_2)$ to be the service time including prediction (cost of prediction as parameter $c_2$) for predicted long jobs ($\ge T$) that is predicted less than $r$

$$\mathbb{E}[S''_{\ge T, r}(c_2)] = \int_{y = 0}^{r} \int_{x = 0}^{\infty} (1 - p_T(x)) \cdot (x + c_2) \cdot g(x,y) dx dy$$ 
$$\mathbb{E}[{S''_{\ge T, r}}^2(c_2)]= \int_{y = 0}^{r} \int_{x = 0}^{\infty} (1 - p_T(x)) \cdot (x + c_2)^2 \cdot g(x,y) dx dy$$

\end{definition}

\begin{lemma}
For \alg{} policy in the server time cost model, the expected mean response time for a predicted long job of true size $x_J$ and predicted size $r$ is

\begin{align*}
\mathbb{E}[T(x_J, r)]_{srv}^{{PL}} = & \frac{\lambda}{2(1-\rho_{r}^{srv})^2} \cdot \Bigg(\mathbb{E}[{S''_{<T}}^2(c_1)] + (c_1 + c_2)^2 \cdot Q(T, r) + \mathbb{E}[{S''_{\ge T, r}}^2(c1 + c_2)]+ a(r)\Bigg) \\
& +\int_{0}^{x_J} \frac{1}{1-\rho_{(r-a)^+}^{srv}} da
\end{align*}

Where $Q(T, r)= \int_{y = r}^{\infty} \int_{x = 0}^{\infty} (1-p_T(x)) \cdot g(x, y) dx dy
\quad (r-a)^+ = max(r-a, 0)$

$$\rho_{r}^{srv} = \lambda \left(\mathbb{E}[S''_{<T}(c_1)] + (c_1 + c_2) \cdot Q(T, r) + \mathbb{E}[S''_{\ge T, r}(c_1+c_2)]  \right)$$ is the load due to jobs of predicted short and jobs predicted long but their service time prediction less than $r$ along with the load of the jobs predictions.

$$a(r) =  \int_{t = r}^{\infty} \int_{x = t-r}^{\infty} (1-p_T(x)) g(x, t) (x - (t - r))^2 \, dx \, dt$$
\end{lemma}

\begin{proof}
Now we calculate $X^{\text{new}}[rank_{d_J}^{\text{worst}} (a)]$, $X_{0}^{\text{old}}[r_{worst}]$ and $X_{i}^{\text{old}}[r_{worst}]$ for predicted long jobs in the server cost most, here again the J's descriptor is ($(b, r) = (0,r)$).

$X^{\text{new}}[rank_{d_J}^{\text{worst}} (a)]$:
$J$'s delay due to $K$ also depends on $K$ is predicted to be short or long.
\[X_{x_K}^{new}[\langle 4, r - a_J \rangle]  = 
\begin{cases}
    c_1 + x_K   & \text{$K$ is predicted short } \\
    c_1 + c_2 + x_K \mathds{1}(r_K < r- a_J) & \text{$K$ is predicted long}
\end{cases}
\]

Employing the joint distribution $g(x,y)$, and setting $\mathds{1}(r_K < r - a_J) = \int_{y=0}^{r - a_J} g(x_K,y) dy$, we can derive $J$'s expected delay due to any random new job:

\begin{align*}
\mathbb{E}[X^{new}[\langle 4, r - a_J \rangle]] = & c_1 + \int_{0}^{\infty} p_T(x) x f(x) dx +  \int_{0}^{\infty} (1-p_T(x)) c_2 f(x) dx\\
& + \int_{0}^{r-a_J} \int_{x = 0}^{\infty} (1 - p_T(x)) x g(x,y) dx dy \\
\end{align*}

$X_{0}^{\text{old}}[r_{worst}]$:
Regardless of $I$'s prediction, an old job has higher priority than $J$, therefore $J$ will be delayed for the duration of $I$'s service.

\[
X_{0, x_I}^{\text{old}}[\langle 4, r \rangle] = 
\begin{cases}
    c_1 + x_I   & \text{if $I$ is predicted short} \\
    c_1 + c_2 + x_I \cdot \mathds{1}(r_I \le r) & \text{if $I$ is predicted long}
\end{cases}
\]

\begin{align*}
\mathbb{E}[(X_{0}^{\text{old}}[\langle 4, r \rangle])]] = &  \int_{0}^{\infty} p_T(x) (c_1 + x) f(x) dx + \int_{y = r}^{\infty} \int_{x = 0}^{\infty} (1-p_T(x)) \cdot g(x, y) \left(c_1 + c_2 \right)  dx dy \\
& + \int_{y = 0}^{r} \int_{x = 0}^{\infty} (1-p_T(x)) \cdot g(x, y) \left(c_1 + c_2 + x\right)  dx dy
\end{align*}

\begin{align*}
\mathbb{E}[(X_{0}^{\text{old}}[\langle 4, r \rangle])^2]] = &  \int_{0}^{\infty} p_T(x) (c_1 + x)^2 f(x) dx + \int_{y = r}^{\infty} \int_{x = 0}^{\infty} (1-p_T(x)) \cdot g(x, y) \left(c_1 + c_2 \right)^2  dx dy \\
& + \int_{y = 0}^{r} \int_{x = 0}^{\infty} (1-p_T(x)) \cdot g(x, y) \left(c_1 + c_2 + x\right)^2  dx dy
\end{align*}

$X_{i}^{\text{old}}[r_{worst}]$:
If $J$ is predicted long, job $I$ may be recycled. This occurs when $I$ is predicted as long and with an expensive prediction $r_I > r$. $I$ is initially considered as \emph{original} and is served during the cheap and expensive prediction phases, then discarded. At age $a_I = r_I - r$, $I$ is recycled and served till completion, which will be $x_I - a_I = x_I - (r_I - r)$. For $i \ge 2$, $X_{i, x_I}^{\text{old}}[\langle 2, r \rangle] = 0$, let $t = r_I$:
\[
X_{1, x_I}^{\text{old}}[\langle 4, r \rangle] =
\begin{cases}
    0   & \text{if $I$ is predicted short} \\
    x_I - (t - r) & \text{if $I$ is predicted long}
\end{cases}
\]

\begin{align*}
\mathbb{E}[X_{1}^{\text{old}}[\langle 4, r \rangle]^2] = \int_{t = r}^{\infty} \int_{x = t-r}^{\infty} (1-p_T(x)) \cdot g(x, t)\cdot(x - (t - r))^2 \cdot dx dt
\end{align*}

Applying Theorem~\ref{soaptheorem} yields the result.

\end{proof}

\begin{lemma}
Let $f_p(y) = \int_{x=0}^{\infty} g(x,y) dx$, the mean response time for a predicted long job with size $x_J$ and any prediction for external cost is
$$\mathbb{E}[T(x_J)_{ext}^{PL}] = \int_{y=0}^{\infty} f_p(y) \mathbb{E}[T(x_J, y)]_{ext}^{{PL}}dy$$
and for server time cost is 
$$\mathbb{E}[T(x_J)_{srv}^{PL}] = \int_{y=0}^{\infty} f_p(y) \mathbb{E}[T(x_J, y)]_{srv}^{{PL}}dy$$

\end{lemma}

\begin{lemma}
The mean response time for predicted long jobs in the external cost model is
$$\mathbb{E}[T_{ext}^{PL}] = \frac{\int_{x=0}^{\infty} \int_{y=0}^{\infty} (1-p_T(x))g(x,y) \mathbb{E}[T(x, y)]_{ext}^{{PL}} dy dx}{\int_{x=0}^{\infty} \int_{y=0}^{\infty} (1-p_T(x))g(x,y) dy dx}$$

and in the server cost model is:

$$\mathbb{E}[T_{srv}^{PL}] = \frac{\int_{x=0}^{\infty} \int_{y=0}^{\infty} (1-p_T(x))g(x,y) \mathbb{E}[T(x, y)]_{srv}^{{PL}} dy dx}{\int_{x=0}^{\infty} \int_{y=0}^{\infty} (1-p_T(x))g(x,y) dy dx}$$

\end{lemma}

\subsection{\alg{} Models Comparison}
In the server cost model, we observe that the mean response times for both predicted short and long jobs are consistently higher than those in the external cost model, because predictions are scheduled on the same server as the jobs. When we set the costs to zero, so there is no cost to predictions, both models yield identical mean response times. This follows from the definitions of $S''_{\ge T, r}(c_2)$ and $S''_{< T}(c_1)$ (Definition~\ref{def:service_time_w_price}) because with zero costs, these definitions match with those of $S''_{\ge T, r}$ and $S''_{< T}$. 
Additionally, setting the threshold $T$ to zero in \alg{} results in SPRPT (Shortest Predicted Remaining Processing Time) scheduling. With this threshold, there are no short jobs (i.e., $\mathbb{E}[S'_{<T}] = 0$), necessitating expensive predictions for all jobs. Thus, the mean response time for the predicted long jobs aligns with the mean response time of SPRPT~\cite{mitzenmacher2019scheduling} (analyzed for both of the two models in Appendix~\ref{appendix:sprpt_proofs}).

We emphasize that while we have derived equations for total costs for both models, comparing the practical implications of these total costs for the two models is challenging. First, resource allocation differs between the models: in the external cost model, predictions are scheduled on a server separate from the one handling the jobs, whereas in the server cost model, both predictions and jobs are scheduled on the same server. Incorporating a fixed cost into the mean response time for the external cost model does not translate directly to service time. This leads to potential differences in the interpretation of the costs of predictions between the two models.

\subsection{Generalization to Non-Fixed Costs}
\label{subsec: non_fixed_costs}
While we assume that the prediction costs as fixed, our approach naturally generalizes theoretically to random costs from a distribution, where that distribution may also depend on the service time of the job. Here we outline the necessary modifications in the analysis for this generalization.

We may consider prediction costs that are assumed to be independent over jobs. The cheap prediction cost is given by a density function $k_1(x,c_1)$, where $k_1(x,c_1)$ is the density corresponding to a job with actual size $x$ and cost of cheap prediction $c_1$. Hence,
$\int_{c=0}^\infty k_1(x,c) dc = f(x)$. Similarly, the expensive prediction cost is given by a density function $k_2(x,c_2)$, where $k_1(x,c_2)$ is the density corresponding to a job with actual size $x$ and the cost of expensive prediction is $c_2$.

In our analysis of the server cost model with fixed costs, for jobs with rank 4 in the first dimension, we used $r-a$ to encode the second dimension of the rank rather than the predicted remaining service time, which we noted is technically $r - (a - c_1 - c_2)$.  When $c_1$ and $c_2$ are fixed, doing so does not change the rankings of jobs, but for non-fixed costs, we would want to use the actual predicted remaining service time $r - (a - c_1 - c_2)$ for the rank function.

\section{Baselines}
\label{sec:baselines}

As baselines, we compare \alg{} with three distinct policies in the two proposed models. These policies are 1) FCFS\footnote{We could also consider any non-size-based policy such as LCFS or FB.}, a non-size-based policy that does not require predictions; 2) SPRPT, which involves performing expensive predictions for all jobs; and 3) 1bit advice~\cite{mitzenmacher2021queues}, which uses only cheap predictions, separating jobs into predicted shorts and predicted longs, and using FCFS as a scheduling policy for each category, meaning that we do not have a second stage of predictions.

These policies, along with \alg{}, can be placed on a spectrum based on their prediction costs. FCFS requires no predictions, while SPRPT requires expensive predictions for each job. The 1bit policy and \alg{} are positioned in the middle of this spectrum, with the 1bit policy incurring lower prediction costs than \alg{}. The question becomes, given prediction costs, what is the most cost-effective policy?

To compare all these policies, we analyze SPRPT and 1bit policies in the external cost model and the server cost model. These policies, initially introduced by Mitzenmacher~\cite{mitzenmacher2019scheduling, mitzenmacher2021queues}, were analyzed without considering the cost of predictions, so here we re-analyze them with prediction costs using the SOAP framework for consistency.

\subsection{SPRPT Analysis}
In SPRPT the job descriptors only include the job size and service time predictions, e.g. $\mathcal{D} =$ [size, predicted time] $= [x, r]$. Thus, the rank function in the external cost model is:
\begin{align}
\label{eq:ext_sprpt_rank_equation}
rank_{ext}([x,r], a) = r - a.
\end{align}
In the server cost model:
\begin{align}
\label{srv:rank_sprpt_equation}
rank_{srv}([x,b], a) = 
\begin{cases}
\langle 1, -a \rangle & \text{if $0 \le a \le c_1$ (initial rank, scheduling prediction)}, \\
\langle 2, r - a \rangle & \text{if } a > c_1 \text{ (jobs after prediction)}. \\
\end{cases}
\end{align}

We provide proofs of the following lemmas in Appendix~\ref{appendix:sprpt_proofs}.

\begin{restatable}{lemma}{lemmasprptext}
\label{lem:sprpt_ext}
    For SPRPT in the external cost model, the expected mean response time for a job of true size $x_J$ and predicted size $r$ is
    
    \begin{align*}
    \mathbb{E}[T(x_J, r)]_{ext}^{{SPRPT}} = &\frac{\lambda}{2(1-\rho'_r)^2} \Bigg(\int_{y = 0}^{r} \int_{x_I = 0}^{\infty} x_I^2\cdot  g(x_I, y) dx_I dy \\
    &+ \int_{t = r}^{\infty} \int_{x_I = t-r}^{\infty} g(x_I, t) \cdot(x_I - (t - r))^2 \cdot dx_I dt\Bigg) + \int_{0}^{x_J} \frac{1}{1-\rho'_{(r-a)^+}} \, da,
    \end{align*} where $\rho'_r= \lambda \int_{y = 0}^{r} \int_{x_I = 0}^{\infty} x_I \cdot g(x_I,y) dx_I dy$.
\end{restatable}

\begin{restatable}{lemma}{lemmasprptsrv}
\label{lem:sprpt_srv}
    For SPRPT in the server time cost model, the expected mean response time for a job of true size $x_J$ and predicted size $r$ is
    \begin{align*}
    \mathbb{E}[T(x_J, r)]_{srv}^{{SPRPT}} = &\frac{\lambda}{{2(1-\rho''_r)^2}} \Bigg(\int_{y = r}^{\infty} \int_{x = 0}^{\infty} c_2^2 \cdot g(x, y)  dx dy + \int_{y = 0}^{r} \int_{x_I = 0}^{\infty} (c_2 + x_I)^2 \cdot g(x_I, y) dx_I dy \\ &+ \int_{t = r}^{\infty} \int_{x_I = t-r}^{\infty} g(x_I, t) \cdot(x_I - (t - r))^2 \cdot dx_I dt)\Bigg)
    + \int_{0}^{x_J} \frac{1}{1-\rho''_{(r-a)^+}} \, da,
    \end{align*} where $\rho''_r= \lambda \left( c_2 + \int_{y = 0}^{r} \int_{x_I = 0}^{\infty} x_I \cdot g(x_I,y) dx_I dy\right)$.
\end{restatable}

\subsection{1bit Analysis}
For the 1bit policy, we define the rank function of this approach and then analyze it using SOAP framework. Here the job descriptors only include the job sizes and binary predictions, e.g. $\mathcal{D} =$ [size, predicted short/long $= [x, b]$. 

\begin{align}
\label{eq:ext_1bit_rank_equation}
rank_{ext}([x,b], a) = 
\begin{cases}
\langle 1, -a \rangle & \text{if } b = 1, \\
\langle 2, - a \rangle &  \text{ if } b = 0.
\end{cases}
\end{align}

In the server cost model, this approach results in the following rank function:
\begin{align}
\label{srv:rank_1bit_equation}
rank_{srv}([x,b], a) = 
\begin{cases}
\langle 2, -a \rangle & \text{if $ 0 \le a \le c_1$ (initial rank, and cheap prediction)}, \\
\langle 1, -a \rangle & \text{if } b = 1 \text{ and } a > c_1 \text{ (short jobs after cheap prediction)}, \\
\langle 3, -a \rangle & \text{if } b = 0 \text{ and } a > c_1 \text{ (long jobs after cheap prediction).}
\end{cases}
\end{align}

The mean response time for predicted short jobs of this approach is similar to predicted short jobs of \alg{} as nothing has changed.

We provide proofs of these 1bit lemmas in Appendix~\ref{appendix:1bit_proofs}.

\begin{restatable}{lemma}{lemmabitext}
\label{lem:1bit_ext}
    For the 1bit policy in the external cost model, the expected mean response time for a predicted long job of true size $x_J$ is
    \begin{align*}
    \mathbb{E}[T(x_J)]_{ext}^{{1bit, PL}} = &\frac{\lambda}{2(1-\rho)(1-\rho^{ext}_{new})} \mathbb{E}[{S}^2] + \int_{0}^{x_J} \frac{1}{1-\rho^{ext}_{new}} \, da,
    \end{align*} where $\rho^{ext}_{new} = \lambda \int_{0}^{\infty} p_T(x) x f(x) dx $, the load due to predicted short jobs.
\end{restatable}

\begin{restatable}{lemma}{lemmabitsrv}
\label{lem:1bit_srv}
    For the 1bit policy in the server time cost model, the expected mean response time for a predicted long job of true size $x_J$ is
    \begin{align*}
        \mathbb{E}[T(x_J)]_{srv}^{{1bit, PL}} = &\frac{\lambda}{2(1-\rho_{c_1})(1-\rho^{srv}_{new})} \int_{0}^{\infty} (x+c_1)^2 f(x) \, dx + \int_{0}^{x_J} \frac{1}{1-\rho^{srv}_{new}} \, da,
    \end{align*} where $\rho^{srv}_{new} = \lambda \left(c_1 + \int_{0}^{\infty} p_T(x) x f(x) dx \right)  \mbox{ and  } \rho_{c_1} = \lambda\left(\int_{0}^{\infty} (x + c_1) f(x) dx \right) $.
\end{restatable}

\subsection{Baselines Comparison}
In the server cost model, the total costs are simply the mean response time in the system for each policy, while in the external model, the total costs are as follows:


{
\renewcommand{\arraystretch}{1.5}
\begin{tabular}{l l}
\textbf{Policy} & \textbf{Cost} \\
\hline
FCFS & $\mathbb{E}[T]^{FCFS} = \frac{\lambda \mathbb{E}[S^2]}{2(1-\rho)} + \mathbb{E}[S]$ \\
SPRPT & \hyperref[lem:sprpt_ext]{$\mathbb{E}[T]_{ext}^{\text{SPRPT}}$}  $ \ + \ c_2$ \\
1bit & $(1-z) \cdot$ \hyperref[lemma_skippredict_PS_ext]{$\mathbb{E}[T]_{ext}^{1bit, PS}$} + $z \cdot$ \hyperref[lem:1bit_ext]{$\mathbb{E}[T]_{ext}^{1bit, PL}$} $\ + \ c_1$ \\
\alg{} & $(1-z) \cdot$ \hyperref[lemma_skippredict_PS_ext]{$\mathbb{E}[T]_{ext}^{\alg{}, PS}$} $+ z \cdot ($\hyperref[lemma_skippredict_PL_ext]{$\mathbb{E}[T]_{ext}^{\alg{}, PL}$} $\ + \ c_2) + c_1$ \\
\end{tabular}
}


Without looking at the formulas, it is intuitive that under even moderate loads, when the cost of the expensive prediction is low (close to $c_1$), the total cost of \alg{} would be greater than that for SPRPT. Also in this case SPRPT would outperform the 1bit approach, as size-based policies are generally more effective than non-size-based ones when predictions are reasonably accurate. However, when the expensive prediction cost $c_2$ is high, \alg{} or 1bit would be better options than SPRPT. Note \alg{} and 1bit have the same response times for jobs predicted to be short, as both schedule these jobs in the same way. Thus, the efficiency of \alg{} over the 1bit approach depends on the value of using an SPRPT-strategy for the remaining jobs.

\section{Simulation with Baselines}
\label{sec:simulation}
To gain more insight into when to invest in prediction and how \alg{} compares to other policies, in this section we compare \alg{}, SPRPT, 1bit and FCFS using simulation.

We consider the setting of the single queue with Poisson arrivals, and two job service time distributions; exponentially distributed with mean 1 ($f(x) = e^{-x}$) and the Weibull distribution with cumulative distribution $F=1-e^{-\sqrt{2x}}$. The Weibull distribution is heavy-tailed, so that while the average service time remains 1, there are many more very long jobs than with the exponential distribution. We use two prediction models: 1) Exponential predictions \cite{mitzenmacher2019scheduling}, where a prediction for a job with service time $x$ is itself exponentially distributed with mean $x$. 2) Uniform predictions \cite{mitzenmacher2019scheduling}, where a prediction for a job with service time $x$ is uniformly distributed between $(1-\alpha)x$ and $(1+\alpha)x$ for a parameter $\alpha$. Table~\ref{tab:prediction_models} in the Appendix~\ref{appendix:simulation} includes $p_T(x)$ and $g(x,y)$ of these models. We note that we have checked simulation results for single queues against the equations we have derived in Section~\ref{sec:results} using the ``perfect predictor,'' where the cheap and expensive predictors are always accurate, as the integrals are simpler in this case. Each of the two-stage predictors could be from a different model. In this section, we present results for the exponentially distributed service time with uniform predictor, where we model both predictors by the uniform model but with different $\alpha$ values; $\alpha = 0.8$ for the cheap predictor and $\alpha = 0.2$ for the expensive one. (The cheap predictor here returns a single bit by comparing the predicted value with the threshold;  this is not how an actual prediction would work, but it is just a test model for simulation.)
For the Weibull distribution, we use the exponential predictor. Results of the other combination appear in Appendix~\ref{appendix:simulation}. The default costs for the external cost model are $c_1 = 0.5, c_2=2 $ and $c_1 = 0.01, c_2=0.05$ for the server cost model. When $T$ is not tested, we default to $T=1$ for simplicity. Though we could aim to pick an optimal threshold for \alg{} for comparison, we have found that $T=1$ is effectively near optimal in many cases, and is suitable for comparison purposes.

\subsection{\alg{} Benefit Increases with Larger Cost Gaps}

We have found that \alg{} is more cost-effective than other policies when there is a difference between the costs of the two predictions, with a greater gap leading to higher improvement.
In Figures~\ref{fig:ext_ratio}~\ref{fig:srv_ratio}~\ref{fig:ext_T_weibull}~\ref{fig:srv_ratio_weibull}, we have $T=1$, $c_1$ is fixed to the default value, and we change $c_2$ by varying $k$ where $c_2=k c_1$. For similar or very close costs ($k=1$), \alg{} is less useful, and SPRPT may be a better option. However, as $c_2$ values increase, \alg{} becomes more cost-effective. FCFS and 1bit are not affected as they do not require expensive predictions. For Weibull distributed service times, the costs of \alg{} and SPRPT are similar due to the heavy-tailed nature of the Weibull distribution. However, as expected, the cost gap grows with an increase in $c_2$.

\begin{figure*}[t]
    \centering
    \begin{tabular}{ccc}
        \subfloat[cost vs $c_2$ - External]{\includegraphics[width = \matrixCellWidth]{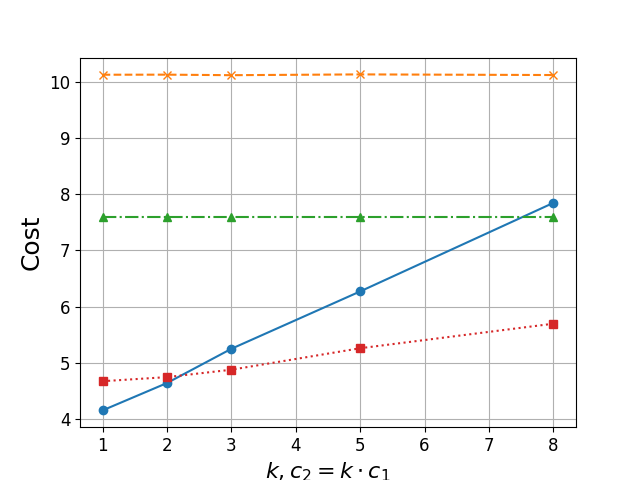}\label{fig:ext_ratio}} &
        \subfloat[cost vs T - External]{\includegraphics[width = \matrixCellWidth]{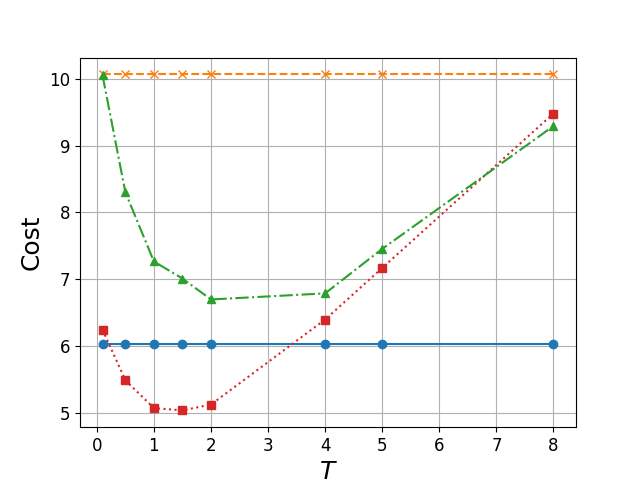}\label{fig:ext_T}} &
        \subfloat[cost vs. $\lambda$ - External]{\includegraphics[width=\matrixCellWidth]{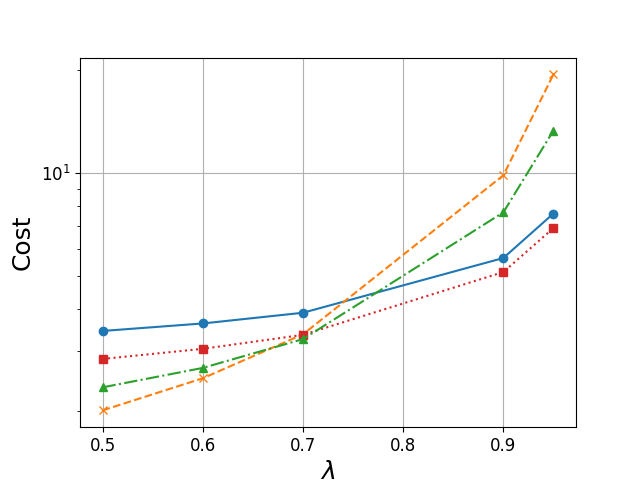}\label{fig:ext_arriving}} \\
        \subfloat[cost vs $c_2$- Server]{\includegraphics[width = \matrixCellWidth]{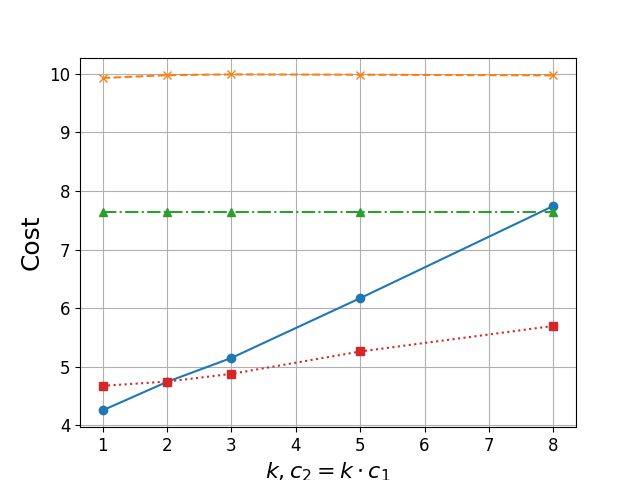}\label{fig:srv_ratio}} &
        \subfloat[cost vs T - Server]{\includegraphics[width = \matrixCellWidth]{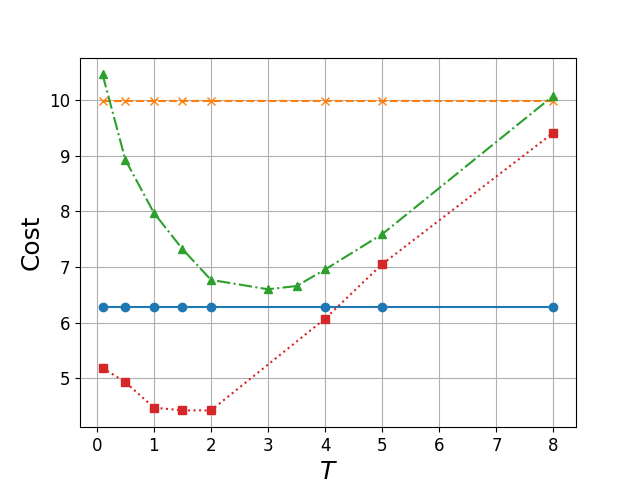}\label{fig:srv_T}} &
        \subfloat[cost vs. $\lambda$- Server]{\includegraphics[width=\matrixCellWidth]{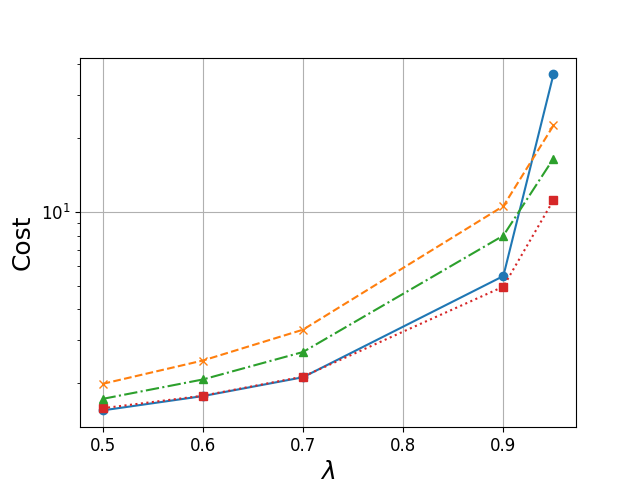}\label{fig:srv_arriving}}
        \tabularnewline
        \multicolumn{3}{c}{\subfloat{\includegraphics[width=0.5\linewidth]{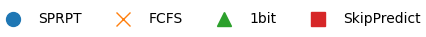}}}	

    \end{tabular}
    \caption{Cost in the external cost and server cost models using uniform predictor, cheap predictor is configured with $\alpha = 0.8$, expensive predictor with $\alpha=0.2$, service times are exponentially distributed with mean 1. The default costs for the external model are $c_1= 0.5, c_2=2$ and for the server cost model are $c_1= 0.01, c_2=0.05$ (a + d) Cost vs. $c_2$ when $\lambda = 0.9$ and $T=1$ (b + e) Cost vs. $T$ when $\lambda = 0.9$ (c + f) Cost vs. $\lambda$ when $T=1$.}
    \label{fig:sim_uniP_dist_exp}
\end{figure*}

\begin{figure*}[t]
    \centering
    \begin{tabular}{ccc}
        \subfloat[cost vs $c_2$ - External]{\includegraphics[width = \matrixCellWidth]{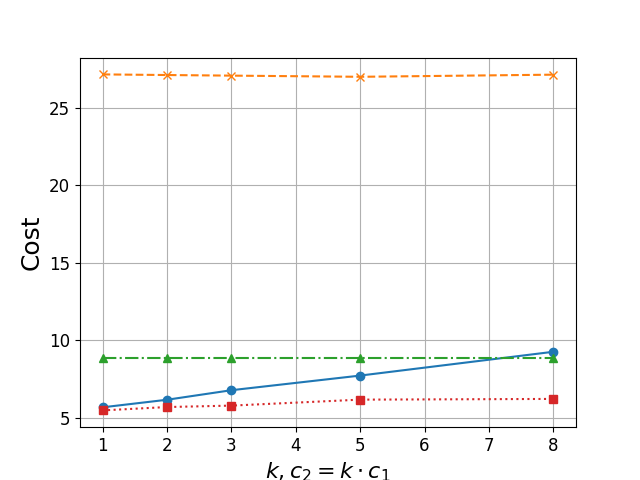}\label{fig:ext_ratio_weibull}} &
        \subfloat[cost vs T - External]{\includegraphics[width = \matrixCellWidth]{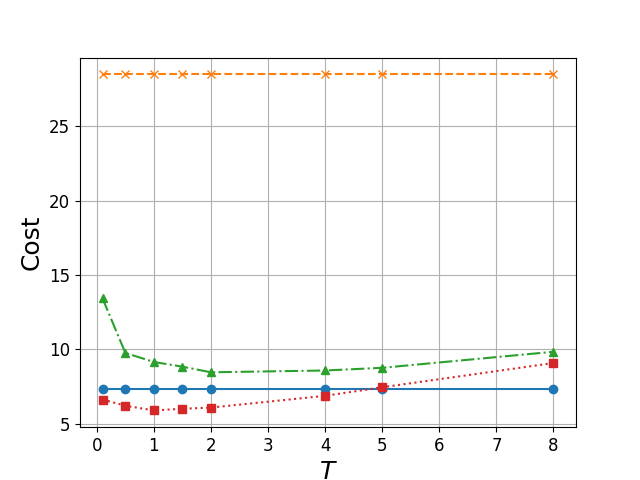}\label{fig:ext_T_weibull}} &
        \subfloat[cost vs. $\lambda$ - External]{\includegraphics[width=\matrixCellWidth]{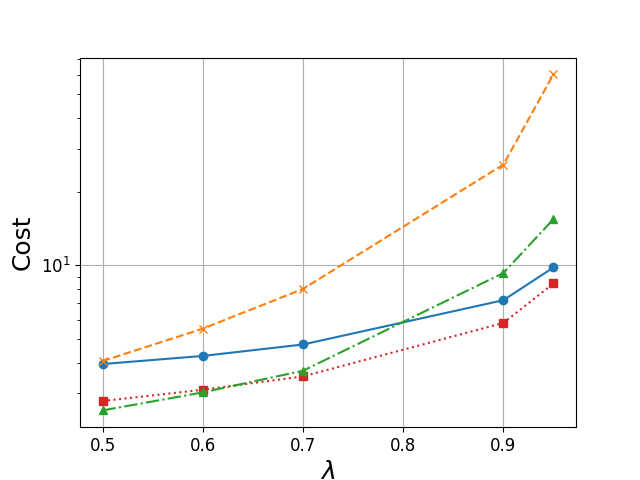}\label{fig:ext_arriving_weibull}} \\
        \subfloat[cost vs $c_2$- Server]{\includegraphics[width = \matrixCellWidth]{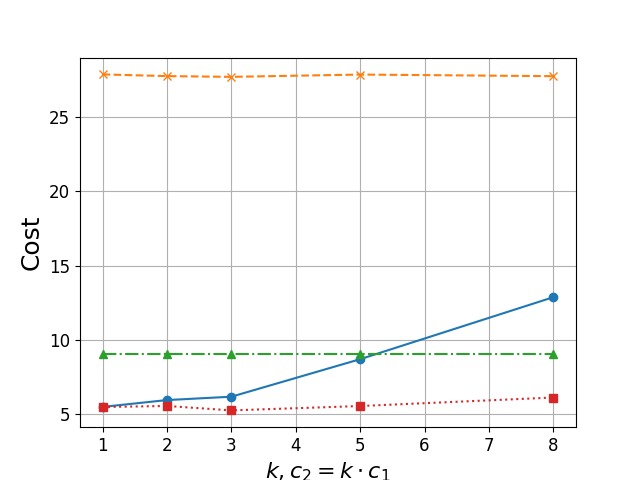}\label{fig:srv_ratio_weibull}} &
        \subfloat[cost vs T - Server]{\includegraphics[width = \matrixCellWidth]{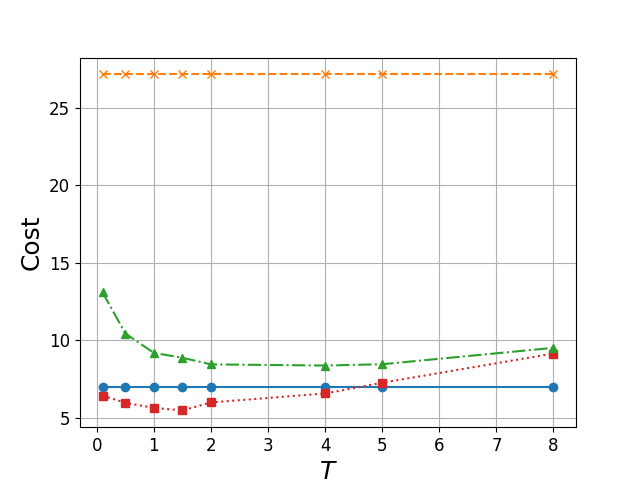}\label{fig:srv_T_weibull}} &
        \subfloat[cost vs. $\lambda$- Server]{\includegraphics[width=\matrixCellWidth]{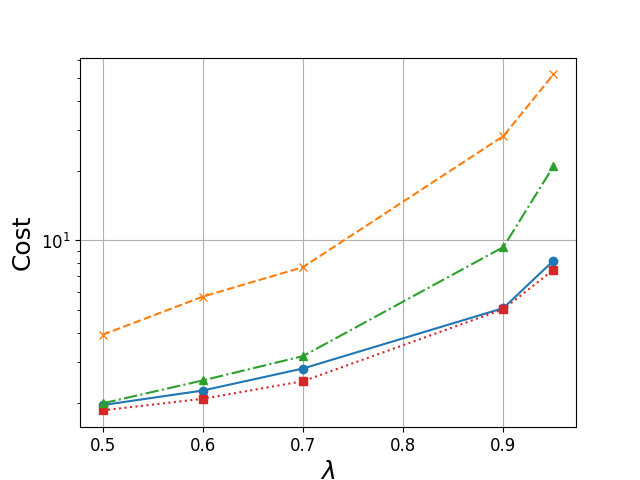}\label{fig:srv_arriving_weibull}}
        \tabularnewline
        \multicolumn{3}{c}{\subfloat{\includegraphics[width=0.5\linewidth]{figs/graphs/legend.png}}}	

    \end{tabular}
    \caption{Cost in the external cost and server cost models using exponential predictor for Weibull service time. The default costs for the external model are $c_1= 0.5, c_2=2$ and for the server cost model are $c_1= 0.01, c_2=0.05$ (a + d) Cost vs. $c_2$ when $\lambda = 0.9$ and $T=1$ (b + e) Cost vs. $T$ when $\lambda = 0.9$ (c + f) Cost vs. $\lambda$ when $T=1$.}
    \label{fig:sim_expP_dist_weibull}
\end{figure*}

\subsection{Expensive Predictions are Effective Under High Load}



At high load, \alg{} has the potential for improvement over other policies. The effectiveness of predictions under low load depends on the service time distribution, as shown in Figures~\ref{fig:ext_arriving} and ~\ref{fig:ext_arriving_weibull}. For exponentially distributed service times, under a low rate, a non-size-based policy (FCFS) is reasonable, while with a Weibull serivce distribution having a cheap prediction, 1bit is preferable. These Figures, with $T=1$, show that investing in expensive predictions (service times) becomes beneficial at higher rates, where \alg{} yields the lowest cost in both distributions.
We should note, however, that under extremely high load, prediction-based scheduling (SPRPT, \alg{} or 1bit) risks overflow issues, making FCFS a better option.

\subsection{Impact of T on \alg{} Cost}
In Figures~\ref{fig:ext_T},~\ref{fig:srv_T},~\ref{fig:ext_T_weibull},~and \ref{fig:srv_T_weibull}, we compare the cost vs. $T$ when $\lambda=0.9$. As $T$ increases, both \alg{} and 1bit demonstrate reduced costs in both service time distributions. 
However, past a certain $T$ threshold, which depends on the arrival rate and actual costs, we observe a rise in costs due to the decreasing number of jobs requiring expensive prediction. This leads to a reduced load for expensive predictions, making job size prediction for scheduling between predicted long jobs less effective. As $T$ gets large, \alg{} and 1bit become the same, as both serve predicted short jobs similarly, which in this case are the majority of the jobs.
When considering Weibull service times and default costs, \alg{} shows a gain over SPRPT but it is smaller than the gain over SPRPT with exponential service time. This is because SPRPT is effective for long jobs in the Weibull distribution. Yet, as shown earlier, \alg{}'s advantages grow with increasing differences in prediction costs.

\section{What if the cheap predictions are not really cheap?}
\label{sec:delayp}

While it is a reasonable assumption in real-world systems that cheap predictions may be available and cost-effective, here we consider the case where cheap predictions are either unavailable or not available for substantially less cost than expensive predictions.  In such a scenario, we expect \alg{} to be less effective, so we here consider an alternative algorithm called \algtwo{}. Unlike the intuitive approach of applying predictions to all jobs, \algtwo{} saves the cheap predictions and avoids doing expensive predictions for all jobs. \algtwo{} schedule jobs FCFS, but instead of using cheap prediction, it limits each job to a given limit $L$ of time, at which point it would be preempted, and treated as a \emph{long job}. At that point, the job could go through expensive prediction and be scheduled based on SPRPT scheduling.
A job that finishes before $L$ units of service would, in this setting, be a short job that finishes without prediction. Similarly to \alg{}, we assume the expensive predictions are independent over jobs and are determined by the density function $g(x,y)$, and that long jobs obtain an expensive prediction at some fixed cost $c_2$.
We define $z' = \int_{x=L}^{\infty} f(x)dx$ as the expected fraction of jobs requiring the expensive prediction and use it in Definitions~\ref{def:external_cost} and~\ref{def:srv_cost} to analyze \algtwo{} cost in the external cost model and in the server cost model. (Note that in \algtwo{} $c_1=0$.)

Rank function of \algtwo{}:
We model the system using $\mathcal{D} =$ [size, predicted time] $= [x, r]$.
Here we assume that the service time prediction $r$ is greater than $L$, because jobs that require expensive prediction are longer than $L$. Since a job's age is $L$ when it is preempted and obtains a prediction, in the external model, the predicted remaining time for long jobs is $r-L - (a-L) = r-a$.  For the service time model their age after being predicted starts at $L+c_2$, so the predicted remaining time is $r-L - (a-L-c_2) = r-a -c_2$. As $c_2$ is fixed among all jobs, instead of the predicted remaining time we can use $r-a$ as the rank for convenience.
Accordingly, \algtwo{} in the external cost model has the following rank function:

\begin{align}
\label{eq:alg2_ext_rank_equation}
rank_{ext}([x, r], a) = 
\begin{cases}
\langle 1, -a \rangle & \text{if }  0 \le a < L \\
\langle 2, r - a \rangle &  \text{ if } a \ge L
\end{cases}
\end{align}

In the server cost model, \algtwo{} results in the following rank function:
\begin{align}
\label{srv:alg2_rank_equation}
rank_{srv}([x, r], a) = 
\begin{cases}
\langle 1, -a \rangle & \text{if $0 \le a < L$ (initial rank)}, \\
\langle 2, -a \rangle & \text{if } L \le a \le L + c_2  \text{ (expensive prediction calculation)}, \\
\langle 3, r-a \rangle & \text{if } a > L + c_2  \text{ (long jobs after expensive prediction)}.
\end{cases}
\end{align}

\begin{restatable}{lemma}{lemmadelaypexts}
\label{lem:delayp_ext_s}
For \algtwo{} in both the external cost model and the server time cost model, the expected mean response time for a short job is
\begin{align*}
    \mathbb{E}[T]_{ext}^{{\algtwo{}, S}} = \mathbb{E}[T]_{srv}^{{\algtwo{}, S}}  = &\frac{\lambda}{2(1-\rho_{L})} \Bigg(\int_{0}^{L} x^2 f(x) \, dx + \int_{L}^{\infty} L^2 f(x) dx \Bigg) + \int_{0}^{L} x f(x) dx,
\end{align*}
where $\rho_{L} = \lambda \left(\int_{0}^{L} x f(x) dx + \int_{L}^{\infty} L f(x) dx\right)$, the load due to jobs while limiting their sizes to $L$.

\end{restatable}

\begin{restatable}{lemma}{lemmadelaypextl}
\label{lem:delayp_ext_l}
    For \algtwo{} in the external cost model, the expected mean response time for a long job of true size $x_J$ and predicted size $r$ is

    \begin{align*}
    \mathbb{E}[T(x_J, r)]_{ext}^{{\algtwo{}, L}}  = &\frac{\lambda}{2(1-\rho^{ext}_{L,r})^2} \Bigg(\int_{x=0}^{L} x^2 f(x) dx  \\
    &+ \int_{y = 0}^{r} \int_{x = L}^{\infty} x^2 \cdot g(x,y) dx dy + \int_{y = r}^{\infty} \int_{x = L}^{\infty} L^2 \cdot g(x,y) dx dy \\
    &+\int_{t = r}^{\infty} \int_{x =L+ t-r}^{\infty} g(x, t)\cdot(x -L- (t - r))^2 \cdot dx dt\Bigg) +  \int_{0}^{x_J} \frac{1}{1-\rho^{ext}_{L,(r-a)^+}} \, da,
\end{align*} where $\rho^{ext}_{L, r} = \lambda \left(\int_{x=0}^{L} x f(x) dx + \int_{y = 0}^{r} \int_{x = L}^{\infty} x \cdot g(x,y) dx dy + \int_{y = r}^{\infty} \int_{x = L}^{\infty} L \cdot g(x,y) dx dy\right)$ is the load due to short jobs, long jobs predicted to be less than $r$, and other long jobs with their size limited at $L$. Here $(r-a)^+ = max(r-a, 0)$.

\end{restatable}

\begin{restatable}{lemma}{lemmadelaypsrvl}
\label{lem:delayp_srv_l}
    For \algtwo{} policy in the server cost model, the expected mean response time for a long job of true size $x_J$ and predicted size $r$ is
\begin{align*}
    \mathbb{E}[T(x_J, r)]_{srv}^{{\algtwo{}, L}}  = &\frac{\lambda}{2(1-\rho^{srv}_{L,r})^2} \Bigg(\int_{x=0}^{L} x^2 f(x) dx + \int_{y = 0}^{r} \int_{x = L}^{\infty} (x+c_2)^2 \cdot g(x,y) dx dy \\
    & + \int_{y = r}^{\infty} \int_{x = L}^{\infty} (L + c_2)^2 \cdot g(x,y) dx dy \\
    &+\int_{t = r}^{\infty} \int_{x =L+ t-r}^{\infty} g(x, t)\cdot(x -L- (t - r))^2 \cdot dx dt\Bigg) +  \int_{0}^{x_J} \frac{1}{1-\rho^{srv}_{L,(r-a)^+}} \, da,
\end{align*} where $\rho^{srv}_{L, r} = \lambda \left(\int_{x=0}^{L} x f(x) dx + \int_{y = 0}^{r} \int_{x = L}^{\infty} (x+c_2) \cdot g(x,y) dx dy + \int_{y = r}^{\infty} \int_{x = L}^{\infty} (L + c_2) \cdot g(x,y) dx dy\right)$
is the load due to short jobs, predictions for long jobs, long jobs predicted to be less than $r$, and other long jobs with their size limited at $L$. Here $(r-a)^+ = max(r-a, 0)$.

\end{restatable}

\begin{figure*}[t]
    \centering
    \begin{tabular}{ccc}
        \subfloat[large cost gap, $c_1=0.5, c_2 = 4$]{\includegraphics[width = \matrixCellWidth]{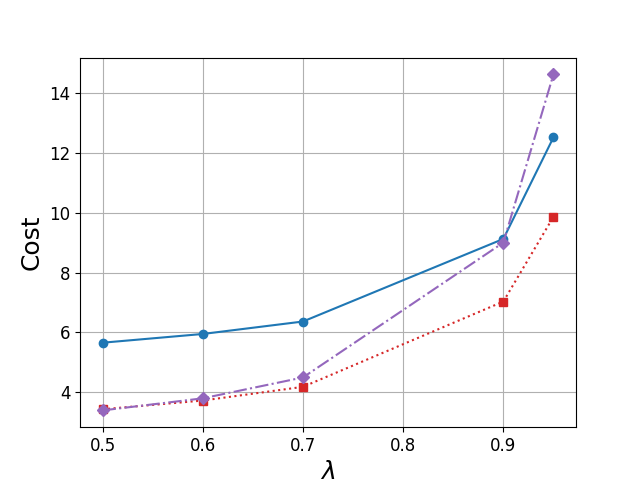}\label{fig:delayp_lowcosts}} &
        \subfloat[small cost gap, $c_1=3.5, c_2 = 4$]{\includegraphics[width = \matrixCellWidth]{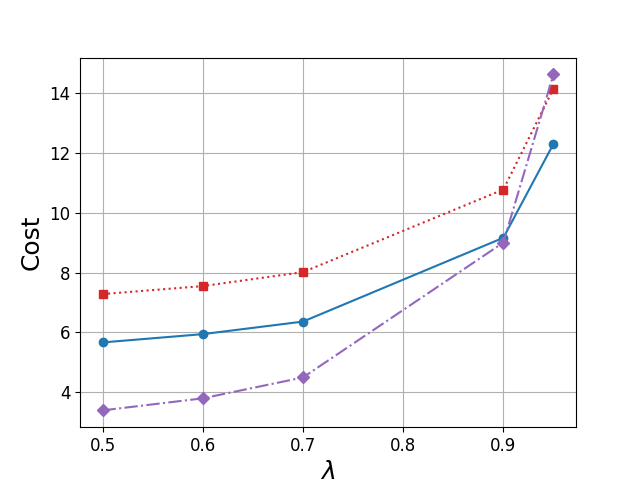}\label{fig:delatp_hightcosts}} &
        \subfloat[cost vs. threshold]{\includegraphics[width=\matrixCellWidth]{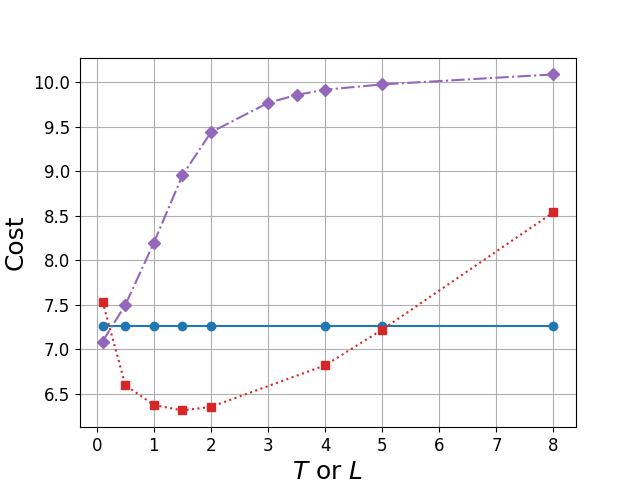}\label{fig:delayp_T}}
        \tabularnewline
        \multicolumn{3}{c}{\subfloat{\includegraphics[width=0.45\linewidth]{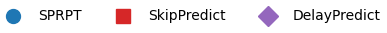}}}	

    \end{tabular}
    \caption{\algtwo{} vs. \alg{} and SPRPT in the external cost model (a) cost vs. arrival rate with $c_1=0.5, c_2=4$, $T=L=1$ (b) cost vs. arrival rate with $c_1=3.5, c_2=4$, $T=L=1$ (c) cost vs. threshold ($T$ for \alg{} and $L$ for \algtwo{}), $c_1=0.5, c_2=2$, $\lambda = 0.9$.}
    \label{fig:delaypredict}
\end{figure*}

We provide proofs of Lemmas~\ref{lem:delayp_ext_s}, \ref{lem:delayp_ext_l} and~\ref{lem:delayp_srv_l} in Appendix~\ref{appendix:delayp_proofs}.

\subsection{\algtwo{} vs. \alg{}}
The main difference between \algtwo{} and \alg{} is in the waiting time. In \alg{}, a predicted short job sees only short jobs while a short job in \algtwo{} sees all short jobs in the queue ahead of it, limited to size $L$. Similarly, in \algtwo{} a long job has to wait behind incoming jobs (including all other long jobs) for at lest time $L$. Thus, while \algtwo{} saves the cheap predictions, its waiting time can be higher than \alg{}.
Figure~\ref{fig:delaypredict} shows a setting where there is a cost gap between prediction costs, and \alg{} outperforms \algtwo{}. However, \algtwo{} still performs better than SPRPT (with $c_1=0.5, c_2=4$). When the costs are close, \algtwo{} is better than both \alg{} and SPRPT, as in this case \alg{} is less effective. Since \algtwo{} waiting time depends on $L$, we see, Figure~\ref{fig:delayp_T}, that as $L$ increases, the cost of \algtwo{} gets higher.  (Note for comparison purposes the $L$ for \algtwo and the $T$ for \alg are chosen to be the same value.)

\section{Variants of \alg{}}
\label{sec:variants}
While here we have focused on analyzing \alg{} in cases where predictions have costs (either external or in server time), we believe that this work is a starting point;  there are many variations of prediction-based scheduling to consider.  We offer some possible variants of \alg{}.  

\begin{itemize}
\item One could analyze systems using separate servers explicitly for prediction. This introduces challenges, depending on the model for the service time for predictions, of dependence between the queues.  
\item Cheap predictions could provide a richer classification than just short or long; one could imagine $k$-bit priorities from cheap predictions and different scheduling for the $2^k$ classes. 
\item Predicting only some of the jobs, say each with some probability, could reduce prediction costs while still providing gains. (Jobs without predictions could be served FCFS with priority between predicted short and long jobs, for example.)
\item Load-based predictions, where predictions are used only when the number of jobs in the queue reaches some limit $L$ (until emptying, or reaches some lower level), seems intuitively useful. When there are many jobs, the gains from ordering the jobs are likely to be higher, to pay the cost of the predictions.  Such schemes are not readily analyzed using the SOAP framework, however, as job descriptors are assumed to be independent of system load.  
\end{itemize}

\section{Conclusion}
\label{sec:discussion}
We have presented \alg{}, the first prediction-based scheduling policy we are aware of that takes into account the cost of prediction. \alg{} is designed for systems where two levels of prediction are available, good (and cheap) and better (but expensive) prediction. While here we have focused on having a binary prediction (short/long) and a prediction of the service time, our framework would also work for other settings.  For example, both the cheap and expensive predictions could be for the service time, with the expensive prediction being a more refined, time-consuming variation of the cheaper process (that even takes the cheap prediction as an input).  
We considered the cost of prediction in scheduling in two models; the external cost model with externally generated predictions, and the server time cost model where predictions require server time and are scheduled alongside jobs. We derived the response time of \alg{} and analyzed total cost formulae in the two cost models for \alg{}. We similarly derived formulae in these models for previously proposed prediction policies where previous analyses ignored prediction costs, namely 1bit and SPRPT, as well as a new policy, \algtwo. We have demonstrated that \alg{} potentially outperforms FCFS, 1bit, SPRPT, and \algtwo in both cost models, especially when there is a significant cost gap between cheap and expensive predictions.

\ifdefined\arXiv
\section*{Acknowledgments}
Rana Shahout was supported in part by Schmidt Futures Initiative and Zuckerman Institute. Michael Mitzenmacher was supported in part by NSF grants CCF-2101140, CNS-2107078, and DMS-2023528.
\fi



\pagebreak

\ifdefined \arvix
\bibliographystyle{plain}
\else
\bibliographystyle{ACM-Reference-Format}
\fi

\bibliography{refs}


\begin{thebibliography}{19}


\ifx \showCODEN    \undefined \def \showCODEN     #1{\unskip}     \fi
\ifx \showDOI      \undefined \def \showDOI       #1{#1}\fi
\ifx \showISBNx    \undefined \def \showISBNx     #1{\unskip}     \fi
\ifx \showISBNxiii \undefined \def \showISBNxiii  #1{\unskip}     \fi
\ifx \showISSN     \undefined \def \showISSN      #1{\unskip}     \fi
\ifx \showLCCN     \undefined \def \showLCCN      #1{\unskip}     \fi
\ifx \shownote     \undefined \def \shownote      #1{#1}          \fi
\ifx \showarticletitle \undefined \def \showarticletitle #1{#1}   \fi
\ifx \showURL      \undefined \def \showURL       {\relax}        \fi
\providecommand\bibfield[2]{#2}
\providecommand\bibinfo[2]{#2}
\providecommand\natexlab[1]{#1}
\providecommand\showeprint[2][]{arXiv:#2}

\bibitem[\protect\citeauthoryear{??}{git}{[n.d.]}]%
        {githubio}
 \bibinfo{year}{[n.d.]}\natexlab{}.
\newblock \bibinfo{title}{Algorithms with Predictions Paper List}.
\newblock
  \bibinfo{howpublished}{\url{https://algorithms-with-predictions.github.io}}.
\newblock


\bibitem[\protect\citeauthoryear{Azar, Leonardi, and Touitou}{Azar
  et~al\mbox{.}}{2021}]%
        {DBLP:conf/stoc/AzarLT21}
\bibfield{author}{\bibinfo{person}{Yossi Azar}, \bibinfo{person}{Stefano
  Leonardi}, {and} \bibinfo{person}{Noam Touitou}.}
  \bibinfo{year}{2021}\natexlab{}.
\newblock \showarticletitle{Flow time scheduling with uncertain processing
  time}. In \bibinfo{booktitle}{\emph{{STOC}}}. \bibinfo{pages}{1070--1080}.
\newblock
\urldef\tempurl%
\url{https://doi.org/10.1145/3406325.3451023}
\showURL{%
\tempurl}


\bibitem[\protect\citeauthoryear{Azar, Leonardi, and Touitou}{Azar
  et~al\mbox{.}}{2022}]%
        {DBLP:conf/soda/AzarLT22}
\bibfield{author}{\bibinfo{person}{Yossi Azar}, \bibinfo{person}{Stefano
  Leonardi}, {and} \bibinfo{person}{Noam Touitou}.}
  \bibinfo{year}{2022}\natexlab{}.
\newblock \showarticletitle{Distortion-Oblivious Algorithms for Minimizing Flow
  Time}. In \bibinfo{booktitle}{\emph{{ACM-SIAM}}}. \bibinfo{pages}{252--274}.
\newblock
\urldef\tempurl%
\url{https://doi.org/10.1137/1.9781611977073.13}
\showURL{%
\tempurl}


\bibitem[\protect\citeauthoryear{Dell'Amico}{Dell'Amico}{2019}]%
        {dell2019scheduling}
\bibfield{author}{\bibinfo{person}{Matteo Dell'Amico}.}
  \bibinfo{year}{2019}\natexlab{}.
\newblock \showarticletitle{Scheduling with inexact job sizes: The merits of
  shortest processing time first}.
\newblock \bibinfo{journal}{\emph{arXiv preprint arXiv:1907.04824}}
  (\bibinfo{year}{2019}).
\newblock


\bibitem[\protect\citeauthoryear{Dell'Amico, Carra, and Michiardi}{Dell'Amico
  et~al\mbox{.}}{2015}]%
        {dell2015psbs}
\bibfield{author}{\bibinfo{person}{Matteo Dell'Amico}, \bibinfo{person}{Damiano
  Carra}, {and} \bibinfo{person}{Pietro Michiardi}.}
  \bibinfo{year}{2015}\natexlab{}.
\newblock \showarticletitle{PSBS: Practical size-based scheduling}.
\newblock \bibinfo{journal}{\emph{IEEE Trans. Comput.}} \bibinfo{volume}{65},
  \bibinfo{number}{7} (\bibinfo{year}{2015}), \bibinfo{pages}{2199--2212}.
\newblock


\bibitem[\protect\citeauthoryear{Grass and Fischer}{Grass and Fischer}{2016}]%
        {grass2016two}
\bibfield{author}{\bibinfo{person}{Emilia Grass} {and} \bibinfo{person}{Kathrin
  Fischer}.} \bibinfo{year}{2016}\natexlab{}.
\newblock \showarticletitle{Two-stage stochastic programming in disaster
  management: A literature survey}.
\newblock \bibinfo{journal}{\emph{Surveys in Operations Research and Management
  Science}} \bibinfo{volume}{21}, \bibinfo{number}{2} (\bibinfo{year}{2016}),
  \bibinfo{pages}{85--100}.
\newblock


\bibitem[\protect\citeauthoryear{Harchol-Balter}{Harchol-Balter}{2013}]%
        {harchol2013performance}
\bibfield{author}{\bibinfo{person}{Mor Harchol-Balter}.}
  \bibinfo{year}{2013}\natexlab{}.
\newblock \bibinfo{booktitle}{\emph{Performance modeling and design of computer
  systems: queueing theory in action}}.
\newblock \bibinfo{publisher}{Cambridge University Press}.
\newblock


\bibitem[\protect\citeauthoryear{Kolbin}{Kolbin}{1977}]%
        {kolbin1977stochastic}
\bibfield{author}{\bibinfo{person}{Viacheslav~Viktorovich Kolbin}.}
  \bibinfo{year}{1977}\natexlab{}.
\newblock \bibinfo{booktitle}{\emph{Stochastic programming}}.
\newblock Number~14. \bibinfo{publisher}{Springer Science \& Business Media}.
\newblock


\bibitem[\protect\citeauthoryear{Mitzenmacher}{Mitzenmacher}{2019}]%
        {mitzenmacher2019scheduling}
\bibfield{author}{\bibinfo{person}{Michael Mitzenmacher}.}
  \bibinfo{year}{2019}\natexlab{}.
\newblock \showarticletitle{Scheduling with predictions and the price of
  misprediction}.
\newblock \bibinfo{journal}{\emph{arXiv preprint arXiv:1902.00732}}
  (\bibinfo{year}{2019}).
\newblock


\bibitem[\protect\citeauthoryear{Mitzenmacher}{Mitzenmacher}{2021}]%
        {mitzenmacher2021queues}
\bibfield{author}{\bibinfo{person}{Michael Mitzenmacher}.}
  \bibinfo{year}{2021}\natexlab{}.
\newblock \showarticletitle{Queues with small advice}. In
  \bibinfo{booktitle}{\emph{SIAM Conference on Applied and Computational
  Discrete Algorithms (ACDA21)}}. SIAM, \bibinfo{pages}{1--12}.
\newblock


\bibitem[\protect\citeauthoryear{Mitzenmacher and Vassilvitskii}{Mitzenmacher
  and Vassilvitskii}{2020}]%
        {DBLP:books/cu/20/MitzenmacherV20}
\bibfield{author}{\bibinfo{person}{Michael Mitzenmacher} {and}
  \bibinfo{person}{Sergei Vassilvitskii}.} \bibinfo{year}{2020}\natexlab{}.
\newblock \showarticletitle{Algorithms with Predictions}.
\newblock In \bibinfo{booktitle}{\emph{Beyond the Worst-Case Analysis of
  Algorithms}}, \bibfield{editor}{\bibinfo{person}{Tim Roughgarden}} (Ed.).
  \bibinfo{publisher}{Cambridge University Press}, \bibinfo{pages}{646--662}.
\newblock
\urldef\tempurl%
\url{https://doi.org/10.1017/9781108637435.037}
\showDOI{\tempurl}


\bibitem[\protect\citeauthoryear{Mitzenmacher and Vassilvitskii}{Mitzenmacher
  and Vassilvitskii}{2022}]%
        {DBLP:journals/cacm/MitzenmacherV22}
\bibfield{author}{\bibinfo{person}{Michael Mitzenmacher} {and}
  \bibinfo{person}{Sergei Vassilvitskii}.} \bibinfo{year}{2022}\natexlab{}.
\newblock \showarticletitle{Algorithms with predictions}.
\newblock \bibinfo{journal}{\emph{Commun. {ACM}}} \bibinfo{volume}{65},
  \bibinfo{number}{7} (\bibinfo{year}{2022}), \bibinfo{pages}{33--35}.
\newblock
\urldef\tempurl%
\url{https://doi.org/10.1145/3528087}
\showDOI{\tempurl}


\bibitem[\protect\citeauthoryear{Salman, Dao, Papadopoulos, Mubeen, and
  Nolte}{Salman et~al\mbox{.}}{2023a}]%
        {salman2023scheduling}
\bibfield{author}{\bibinfo{person}{Shaik~Mohammed Salman},
  \bibinfo{person}{Van-Lan Dao}, \bibinfo{person}{Alessandro~Vittorio
  Papadopoulos}, \bibinfo{person}{Saad Mubeen}, {and} \bibinfo{person}{Thomas
  Nolte}.} \bibinfo{year}{2023}\natexlab{a}.
\newblock \showarticletitle{Scheduling firm real-time applications on the edge
  with single-bit execution time prediction}. In
  \bibinfo{booktitle}{\emph{ISORC}}. \bibinfo{pages}{207--213}.
\newblock


\bibitem[\protect\citeauthoryear{Salman, Papadopoulos, Mubeen, and
  Nolte}{Salman et~al\mbox{.}}{2023b}]%
        {salman2023evaluating}
\bibfield{author}{\bibinfo{person}{Shaik~Mohammed Salman},
  \bibinfo{person}{Alessandro~Vittorio Papadopoulos}, \bibinfo{person}{Saad
  Mubeen}, {and} \bibinfo{person}{Thomas Nolte}.}
  \bibinfo{year}{2023}\natexlab{b}.
\newblock \showarticletitle{Evaluating Dispatching and Scheduling Strategies
  for Firm Real-Time Jobs in Edge Computing}. In
  \bibinfo{booktitle}{\emph{IECON}}. \bibinfo{pages}{1--6}.
\newblock


\bibitem[\protect\citeauthoryear{Scully, Grosof, and Mitzenmacher}{Scully
  et~al\mbox{.}}{2022}]%
        {DBLP:conf/innovations/ScullyGM22}
\bibfield{author}{\bibinfo{person}{Ziv Scully}, \bibinfo{person}{Isaac Grosof},
  {and} \bibinfo{person}{Michael Mitzenmacher}.}
  \bibinfo{year}{2022}\natexlab{}.
\newblock \showarticletitle{Uniform Bounds for Scheduling with Job Size
  Estimates}. In \bibinfo{booktitle}{\emph{{ITCS}}}.
  \bibinfo{pages}{114:1--114:30}.
\newblock
\urldef\tempurl%
\url{https://doi.org/10.4230/LIPIcs.ITCS.2022.114}
\showURL{%
\tempurl}


\bibitem[\protect\citeauthoryear{Scully and Harchol-Balter}{Scully and
  Harchol-Balter}{2018}]%
        {scully2018soap}
\bibfield{author}{\bibinfo{person}{Ziv Scully} {and} \bibinfo{person}{Mor
  Harchol-Balter}.} \bibinfo{year}{2018}\natexlab{}.
\newblock \showarticletitle{SOAP bubbles: Robust scheduling under adversarial
  noise}. In \bibinfo{booktitle}{\emph{2018 56th Annual Allerton Conference on
  Communication, Control, and Computing (Allerton)}}.
  \bibinfo{pages}{144--154}.
\newblock


\bibitem[\protect\citeauthoryear{Shmoys and Swamy}{Shmoys and Swamy}{2004}]%
        {shmoys2004stochastic}
\bibfield{author}{\bibinfo{person}{David~B Shmoys} {and}
  \bibinfo{person}{Chaitanya Swamy}.} \bibinfo{year}{2004}\natexlab{}.
\newblock \showarticletitle{Stochastic optimization is (almost) as easy as
  deterministic optimization}. In \bibinfo{booktitle}{\emph{45th Annual IEEE
  Symposium on Foundations of Computer Science}}. IEEE,
  \bibinfo{pages}{228--237}.
\newblock


\bibitem[\protect\citeauthoryear{Swamy and Shmoys}{Swamy and Shmoys}{2006}]%
        {swamy2006approximation}
\bibfield{author}{\bibinfo{person}{Chaitanya Swamy} {and}
  \bibinfo{person}{David~B Shmoys}.} \bibinfo{year}{2006}\natexlab{}.
\newblock \showarticletitle{Approximation algorithms for 2-stage stochastic
  optimization problems}.
\newblock \bibinfo{journal}{\emph{ACM SIGACT News}} \bibinfo{volume}{37},
  \bibinfo{number}{1} (\bibinfo{year}{2006}), \bibinfo{pages}{33--46}.
\newblock


\bibitem[\protect\citeauthoryear{Wierman and Nuyens}{Wierman and
  Nuyens}{2008}]%
        {wierman2008scheduling}
\bibfield{author}{\bibinfo{person}{Adam Wierman} {and} \bibinfo{person}{Misja
  Nuyens}.} \bibinfo{year}{2008}\natexlab{}.
\newblock \showarticletitle{Scheduling despite inexact job-size information}.
  In \bibinfo{booktitle}{\emph{Proceedings of the 2008 ACM SIGMETRICS
  international conference on Measurement and modeling of computer systems}}.
  \bibinfo{pages}{25--36}.
\newblock


\end{thebibliography}

\appendix
\appendix
\section{SPRPT Proofs}
\label{appendix:sprpt_proofs}

\lemmasprptext*

\begin{proof}
    
SPRPT has a rank function $rank(r, a) = r - a$ for job of size $x$ and predicted size $r$. As this rank function is monotonic, $J$'s worst future rank is its initial prediction:
\[ rank_{d_J, x_J}^{\text{worst}} (a)= r -a.\] When $a_J=0$, the rank function is denoted by $r_{worst}= rank_{d_J, x_J}^{\text{worst}} (0) = r$.

\subsubsection{$X^{\text{new}}[rank_{d_J, x_J}^{\text{worst}} (a)]$ computation:}
Suppose that a new job $K$ of predicted size $r_K$ arrives when $J$ has age $a$. 
If $K$ has a predicted job size less than $J$'s predicted remaining process time ($r - a$), $K$ will always outrank $J$. Thus

\[X_{x_K}^{new}[r - a]  = x_K \mathds{1}(r_K < r - a)\]

\begin{align*}
\mathbb{E}[X^{new}[r - a]] = \int_{0}^{r-a} \int_{x_K = 0}^{\infty} x_K \cdot g(x_K,y) dx_K dy \\
\end{align*}

\subsubsection{$X_{0}^{\text{old}}[r_{worst}]$ computation:}
Whether job $I$ is an original or recycled job depends on its predicted size relative to J's predicted size.
If $r_I \le r$, then $I$ is original until its completion because its rank never exceeds $r$.

\[
X_{0, x_I}^{\text{old}}[r] = x_I\mathds{1}(r_I \le r).
\]

\begin{align*}
\mathbb{E}[X_{0}^{\text{old}}[r]] = \int_{y = 0}^{r} \int_{x_I = 0}^{\infty} x_I \cdot g(x_I,y) dx_I dy.
\end{align*}

\begin{align*}
\mathbb{E}[(X_{0}^{\text{old}}[r])^2]] = \int_{y = 0}^{r} \int_{x_I = 0}^{\infty} x_I^2 \cdot g(x_I, y) dx_I dy.
\end{align*}

\subsubsection{$X_{i}^{\text{old}}[r_{worst}]$ computation:}
If $r_I > r$, then $I$ starts discarded but becomes recycled when $r_I - a = r$. This means at age $a = r_I -r$ and served till completion, which will be $x_I - a_I = x_I - (r_I - r)$,  let $t = r_I$:

Thus, we have
\[
X_{1, x_I}^{\text{old}}[r] = x_I - (t - r).
\]
For $i \ge 2$, $$X_{i, x_I}^{\text{old}}[r] = 0.$$

\begin{align*}
\mathbb{E}[X_{1}^{\text{old}}[r]^2] = \int_{t = r}^{\infty} \int_{x_I = t-r}^{\infty} g(x_I, t) \cdot (x_I - (t - r))^2 \cdot dx_I dt.
\end{align*}

Applying Theorem 5.5 of SOAP~\cite{scully2018soap} yields that the mean response time of jobs with descriptor ($r$) and size $x_J$ is as follows.
Let
$$\rho'_r= \lambda \int_{y = 0}^{r} \int_{x_I = 0}^{\infty} x_I \cdot g(x_I,y) dx_I dy.$$
Then

\begin{align*}
\mathbb{E}[T(x_J, r)]_{ext}^{{SPRPT}} = &\frac{\lambda \left(\int_{y = 0}^{r} \int_{x_I = 0}^{\infty} x_I^2\cdot  g(x_I, y) dx_I dy + \int_{t = r}^{\infty} \int_{x_I = t-r}^{\infty} g(x_I, t) \cdot(x_I - (t - r))^2 \cdot dx_I dt\right)}{2(1-\rho'_r)^2} \\
&+ \int_{0}^{x_J} \frac{1}{1-\rho'_{(r-a)^+}} \, da.
\end{align*}

Let $f_p(y) = \int_{x=0}^{\infty} g(x,y) dx$. Then the mean response time for a job with size $x_J$, and the mean response time over all jobs are given by 
$$\mathbb{E}[T(x_J)] = \int_{y=0}^{\infty} f_p(y) \mathbb{E}[T(x_J,y)] dy,$$

$$\mathbb{E}[T]_{ext}^{{SPRPT}} = \int_{x=0}^{\infty} \int_{y=0}^{\infty} g(x,y)  \mathbb{E}[T(x,y)] dy dx.$$

\end{proof}

\lemmasprptsrv*

\begin{proof}

$X^{\text{new}}[rank_{d_J}^{\text{worst}} (a)]$:
$J$'s worst future rank is $\langle 2, r - a_J\rangle$. In this case, $J$'s delay due to a new job $K$ is $c_2$ plus $x_K$ is its predicted service time less than $J$'s remaining process time.
\[X_{x_K}^{new}[\langle 2, r - a_J\rangle]  = 
c_2 + x_K \mathds{1}(r_K < r - a_J).
\]

\begin{align*}
\mathbb{E}[X^{new}[\langle 2, r - a_J\rangle] = c_2 + \int_{0}^{r-a} \int_{x_K = 0}^{\infty} x_K \cdot g(x_K,y) dx_K dy. \\
\end{align*}

\subsubsection{$X_{0}^{\text{old}}[r_{worst}]$ computation:}
In this model, old job $I$ delays $J$ at least $c_2$. In addition,
If $r_I \le r$, then $I$ is original until its completion because its rank never exceeds $r$.

\[
X_{0, x_I}^{\text{old}}\langle 2, r - a_J\rangle = c_2 + x_I\mathds{1}(r_I \le r).
\]

\begin{align*}
\mathbb{E}[X_{0}^{\text{old}}\langle 2, r - a_J\rangle] = c_2 + \int_{y = 0}^{r} \int_{x_I = 0}^{\infty} x_I \cdot g(x_I,y) dx_I dy. 
\end{align*}

\begin{align*}
\mathbb{E}[(X_{0}^{\text{old}}\langle 2, r - a_J\rangle)^2]] =  \int_{y = r}^{\infty} \int_{x = 0}^{\infty} c_2^2 \cdot g(x, y)  dx dy + \int_{y = 0}^{r} \int_{x_I = 0}^{\infty} (c_2 + x_I)^2 \cdot g(x_I, y) dx_I dy.
\end{align*}

\subsubsection{$X_{i}^{\text{old}}[r_{worst}]$ computation:}
If $r_I > r$, then $I$ starts discarded but becomes recycled when $r_I - a = r$. This means at age $a = r_I -r$ and served till completion, which will be $x_I - a_I = x_I - (r_I - r)$,  let $t = r_I$:

Thus, we have
\[
X_{1, x_I}^{\text{old}}\langle 2, r - a_J\rangle = x_I - (t - r).
\]
For $i \ge 2$, $$X_{i, x_I}^{\text{old}}\langle 2, r - a_J\rangle = 0,$$

\begin{align*}
\mathbb{E}[X_{1}^{\text{old}}\langle 2, r - a_J\rangle^2] = \int_{t = r}^{\infty} \int_{x_I = t-r}^{\infty} g(x_I, t) \cdot (x_I - (t - r))^2 \cdot dx_I dt.
\end{align*}

Applying Theorem 5.5 of SOAP~\cite{scully2018soap} yields that the mean response time of jobs with descriptor ($r$) and size $x_J$ is as follows.
Let
$$\rho''_r= \lambda \left( c_2 + \int_{y = 0}^{r} \int_{x_I = 0}^{\infty} x_I \cdot g(x_I,y) dx_I dy\right).$$
Then

\begin{align*}
\mathbb{E}[T(x_J, r)]_{srv}^{{SPRPT}} = &\frac{\lambda}{{2(1-\rho''_r)^2}} \Bigg(\int_{y = r}^{\infty} \int_{x = 0}^{\infty} c_2^2 \cdot g(x, y)  dx dy + \int_{y = 0}^{r} \int_{x_I = 0}^{\infty} (c_2 + x_I)^2 \cdot g(x_I, y) dx_I dy \\ &+ \int_{t = r}^{\infty} \int_{x_I = t-r}^{\infty} g(x_I, t) \cdot(x_I - (t - r))^2 \cdot dx_I dt)\Bigg)
+ \int_{0}^{x_J} \frac{1}{1-\rho''_{(r-a)^+}} \, da
\end{align*}

Let $f_p(y) = \int_{x=0}^{\infty} g(x,y) dx$. The mean response time for a job with size $x_J$ and the mean time for a general job are give by
$$\mathbb{E}[T(x_J)] = \int_{y=0}^{\infty} f_p(y) \mathbb{E}[T(x_J,y)] dy,$$

$$\mathbb{E}[T]_{ext}^{{SPRPT}} = \int_{x=0}^{\infty} \int_{y=0}^{\infty} g(x,y)  \mathbb{E}[T(x,y)] dy dx.$$

\end{proof}

\section{1bit Proofs}
\label{appendix:1bit_proofs}

\lemmabitext*
\begin{proof}
To analyze 1-bit advice for predicted long job in the external cost model using SOAP, we first find the worst future rank and then calculate $X^{\text{new}}[rank_{d_J}^{\text{worst}} (a)]$, $X_{0}^{\text{old}}[r_{worst}]$ and $X_{i}^{\text{old}}[r_{worst}]$ for predicted long job.
For predicted long jobs, the rank function is monotonic as described in \eqref{eq:ext_1bit_rank_equation}, therefore J's worst future rank is its initial rank.
\[ rank_{d_J}^{\text{worst}} (a)=  \langle 2, -a \rangle, \] and  $r_{worst}= rank_{d_J}^{\text{worst}} (0) = \langle 2, 0 \rangle$.

$X^{\text{new}}[rank_{d_J}^{\text{worst}} (a)]$:
Suppose that a new job $K$ arrives when $J$ has age $a_J$.
$J$'s delay due to $K$ depends on $K$ is predicted to be short or long. 

Only if $K$ is predicted short then it will preempt $J$ and be scheduled till completion because it has a higher class as long jobs are scheduled also according to FCFS so in case of $K$ predicted long it will not preempt $J$

\[X_{x_K}^{new}[\langle 2, - a \rangle]  = 
\begin{cases}
    x_K   & \text{$K$ is predicted short. } \\
\end{cases}
\]

\begin{align*}
\mathbb{E}[X^{new}[\langle 2,- a \rangle]] = & \int_{0}^{\infty} p_T(x) x f(x) dx.
\end{align*}

$X_{0}^{\text{old}}[r_{worst}]$:
Now if old job $I$ either predicted short or long is an original job then it remains original until its completion (regardless of if it is predicted long or short). Thus, for $i \ge 1$, $X_{i, x_I}^{\text{old}}[\langle 2, 0 \rangle] = 0$.

\[
X_{0, x_I}^{\text{old}}[\langle 2, 0 \rangle] = x_I.
\]

\begin{align*}
\mathbb{E}[X_{0}^{\text{old}}[\langle 2, 0 \rangle]] = & \int_{0}^{\infty} x f(x) dx.
\end{align*}

\begin{align*}
\mathbb{E}[(X_{0}^{\text{old}}[\langle 2, 0 \rangle])^2]] = &  \int_{0}^{\infty} x^2 f(x) dx.
\end{align*}

Applying Theorem 5.5 of SOAP~\cite{scully2018soap} yields the result:
\begin{align*}
    \mathbb{E}[T(x_J)]_{ext}^{{PL}} = &\frac{\lambda}{2(1-\rho)(1-\rho^{ext}_{new})} \int_{0}^{\infty} x^2 f(x) \, dx + \int_{0}^{x_J} \frac{1}{1-\rho^{ext}_{new}} \, da,
\end{align*}
where $\rho^{ext}_{new} = \lambda \left(\int_{0}^{\infty} p_T(x) x f(x) dx \right)$.

As a second way of thinking about this proof, it can be said that this is the original FCFS system with slowdowns caused by subsystem 1 (predicted short jobs) which is $\frac{1}{1-\rho^{ext}_{new}}$.

\end{proof}

\lemmabitsrv*

\begin{proof}

In this case, according to \eqref{srv:rank_1bit_equation}, J's worst future rank is
\[ rank_{d_J}^{\text{worst}} (a)=  \langle 3, -a \rangle \] and  $r_{worst}= rank_{d_J}^{\text{worst}} (0) = \langle 3, 0 \rangle$.

$X^{\text{new}}[rank_{d_J}^{\text{worst}} (a)]$:
Let's say $J$ has age $a_J$ when $K$ arrives. The delay caused by $K$ for $J$ depends on whether it is predicted to be short or long for $K$. If $K$ is predicted short then it will preempt $J$ and be scheduled along with its cheap prediction till completion. Otherwise, if $K$ is predicted long, it will delay $J$ only for the cheap prediction.

\[X_{x_K}^{new}[\langle 3, - a_J \rangle]  = 
\begin{cases}
    c_1 + x_K   & \text{$K$ is predicted short, } \\
    c_1 & \text{$K$ is predicted long. }
\end{cases}
\]

\begin{align*}
\mathbb{E}[X^{new}[\langle 3,- a \rangle]] = & c_1 + \int_{0}^{\infty} p_T(x) x f(x) dx.
\end{align*}

$X_{0}^{\text{old}}[r_{worst}]$:
Each old job is scheduled for cheap prediction (which costs $c_1$), and an old job $I$, regardless of whether it is predicted long or short remains original until its completion. Thus, for $i \ge 1$, $X_{i, x_I}^{\text{old}}[\langle 3, 0 \rangle] = 0$.

\[
X_{0, x_I}^{\text{old}}[\langle 3, 0 \rangle] = c_1 + x_I
\]

\begin{align*}
\mathbb{E}[X_{0}^{\text{old}}[\langle 3, 0 \rangle]] = & \int_{0}^{\infty} (c_1 + x) f(x) dx
\end{align*}

\begin{align*}
\mathbb{E}[(X_{0}^{\text{old}}[\langle 3, 0 \rangle])^2]] = &  \int_{0}^{\infty} (c_1 + x)^2 f(x) dx
\end{align*}

Using Theorem 5.5 of SOAP~\cite{scully2018soap} yields the result:
\begin{align*}
    \mathbb{E}[T(x_J)]_{srv}^{{1bit, PL}} = &\frac{\lambda}{2(1-\rho_{c_1})(1-\rho^{srv}_{new})} \int_{0}^{\infty} (x+c_1)^2 f(x) \, dx + \int_{0}^{x_J} \frac{1}{1-\rho^{srv}_{new}} \, da,
\end{align*}
where $\rho^{srv}_{new} = \lambda \left(c_1 + \int_{0}^{\infty} p_T(x) x f(x) dx \right) \quad \rho_{c_1} = \lambda\left(\int_{0}^{\infty} (x + c_1) f(x) dx \right) $.

\end{proof}

\section{Simulation}
\label{appendix:simulation}

Here we present some additional simulations of the exponentially distributed service time with the exponential predictor and of the Weibull distribution using the uniform predictor. Table~\ref{tab:prediction_models} summarizes the quantities $p_T(x)$ and $g(x,y)$ for our prediction models.

\begin{table}[h]
\centering
\begin{tabular}{l l l}
\hline
\textbf{Model} & \( p_T(x) \) & \( g(x,y) \) \\
\hline
Perfect Prediction & \( 1 \text{ if } x < T \) & \( e^{-x} \) \\
Exponential Prediction & \( 1 - e^{-(\frac{T}{x})} \) & \( e^{\frac{-x-y}{x}} \) \\
Uniform Prediction & \( \begin{cases} 
0 & \text{if } T \leq (1 - \alpha) x, \\
1 & \text{if } T \geq (1 + \alpha) x, \\
\frac{T - (1 - \alpha) x}{2 \alpha x} & \text{otherwise}
\end{cases} \) & \( \frac{1}{2\alpha x}e^{-x} \) \\
\hline
\end{tabular}
\caption{Prediction Models and their Functions}
\label{tab:prediction_models}
\end{table}

\begin{figure}[H]
    \centering
    \begin{tabular}{ccc}
        \subfloat[cost vs $c_2$ - External]{\includegraphics[width = \matrixCellWidth]{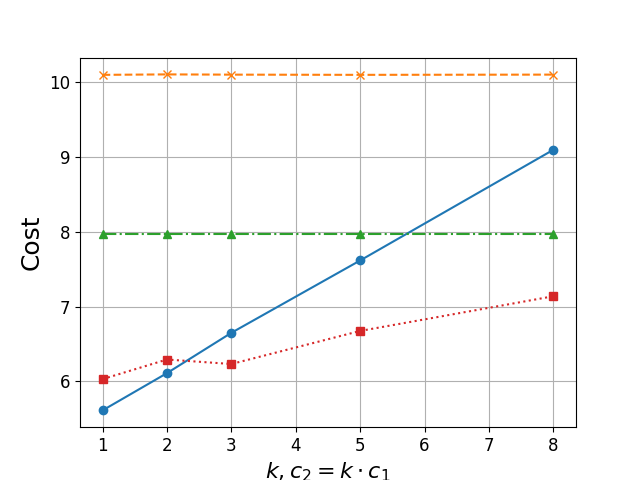}} &
        \subfloat[cost vs T - External]{\includegraphics[width = \matrixCellWidth]{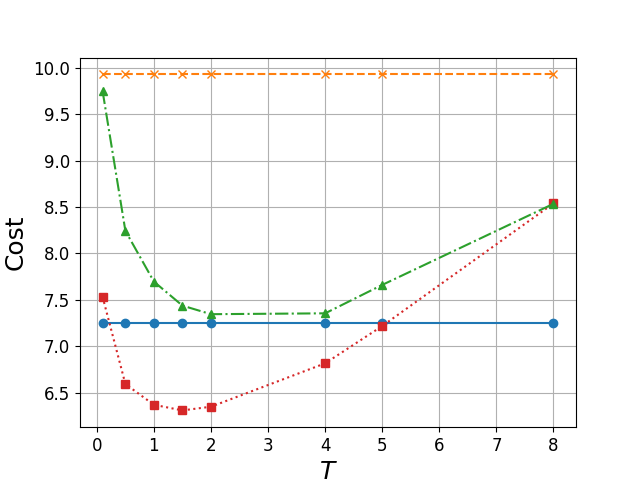}\label{fig:ext_T_expP}} &
        \subfloat[cost vs. $\lambda$ - External]{\includegraphics[width=\matrixCellWidth]{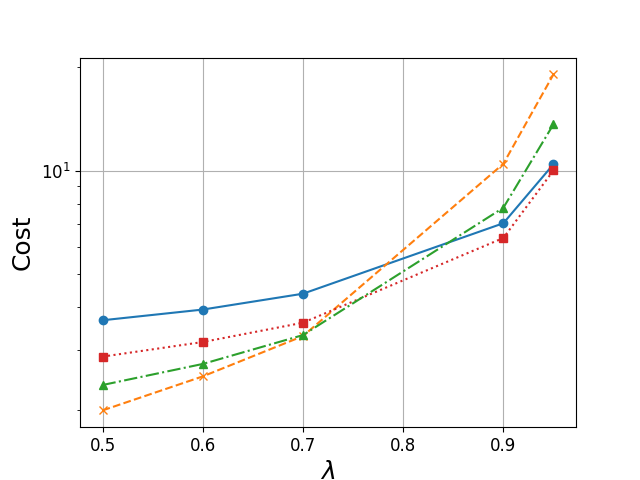}\label{fig:ext_arriving_expP}} \\
        \subfloat[cost vs $c_2$- Server]{\includegraphics[width = \matrixCellWidth]{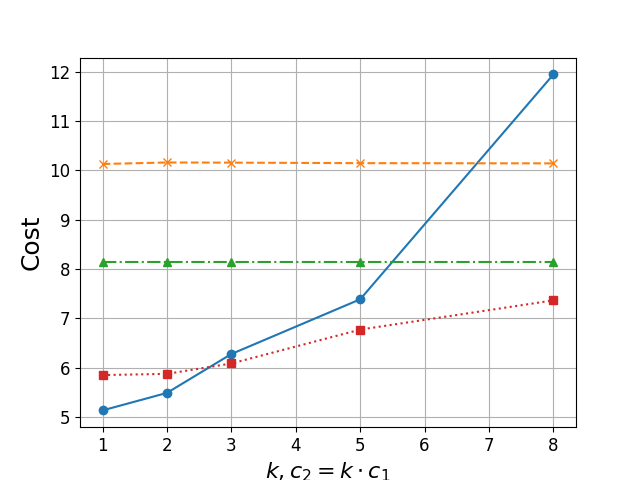}\label{fig:srv_ratio_expP}} &
        \subfloat[cost vs T - Server]{\includegraphics[width = \matrixCellWidth]{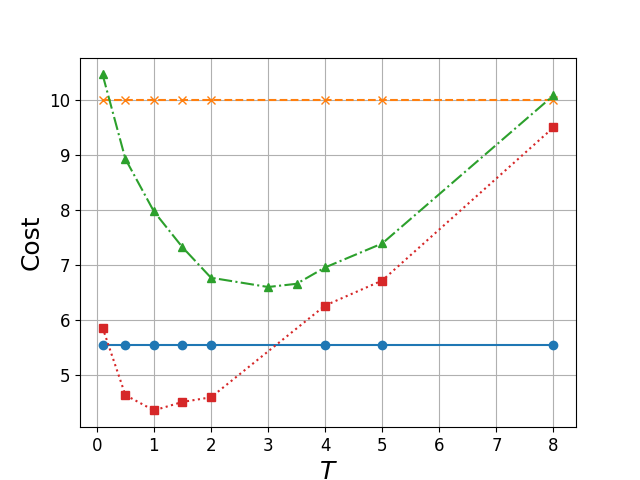}\label{fig:srv_T_expP}} &
        \subfloat[cost vs. $\lambda$- Server]{\includegraphics[width=\matrixCellWidth]{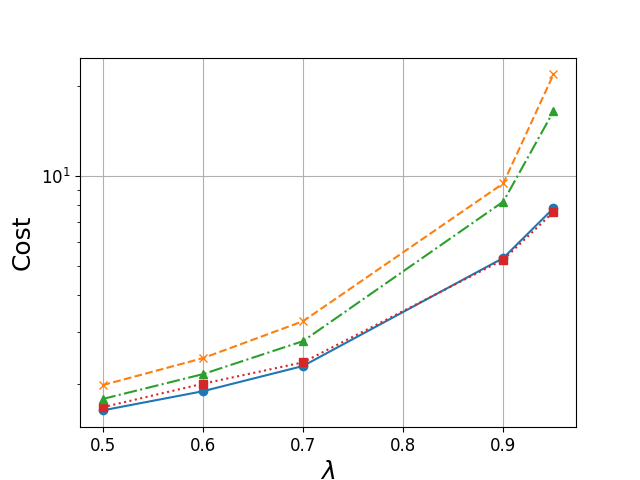}\label{fig:srv_arriving_expP}}
        \tabularnewline
        \multicolumn{3}{c}{\subfloat{\includegraphics[width=0.5\linewidth]{figs/graphs/legend.png}}}	

    \end{tabular}
    \caption{Cost in the external cost and server cost models using exponential predictor for both cheap and expensive predictors, service times are distributed exponentially with mean $1$. The default costs for the external model are $c_1= 0.5, c_2=2$ and for the server cost model are $c_1= 0.01, c_2=0.05$ (a + d) Cost vs. $c_2$ when $\lambda = 0.9$ and $T=1$ (b + e) Cost vs. $T$ when $\lambda = 0.9$ (c + f) Cost vs. $\lambda$ when $T=1$.}
    \label{fig:sim_expP_dist_exp}
\end{figure}

\begin{figure}[H]
    \centering
    \begin{tabular}{ccc}
        \subfloat[cost vs $c_2$ - External]{\includegraphics[width = \matrixCellWidth]{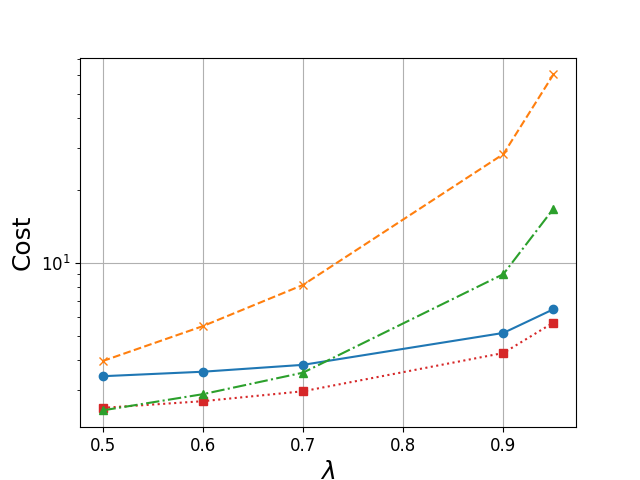}\label{fig:ext_ratio_weibull_uniform}} &
        \subfloat[cost vs T - External]{\includegraphics[width = \matrixCellWidth]{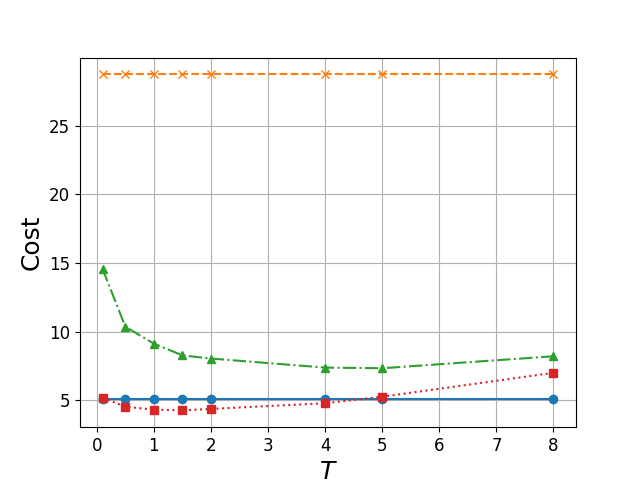}\label{fig:ext_T_weibull_uniform}} &
        \subfloat[cost vs. $\lambda$ - External]{\includegraphics[width=\matrixCellWidth]{figs/graphs/model1/weibull_dist/cost_vs_arrivalrate_2p_uniP_T_1.0_c1_0.5_c2_2_weilbull.png}\label{fig:ext_arriving_weibull_uniform}} \\
        \subfloat[cost vs $c_2$- Server]{\includegraphics[width = \matrixCellWidth]{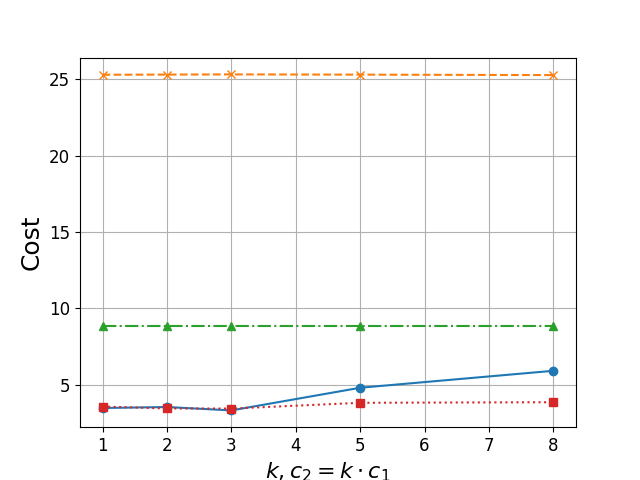}\label{fig:srv_ratio_weibull_uniform}} &
        \subfloat[cost vs T - Server]{\includegraphics[width = \matrixCellWidth]{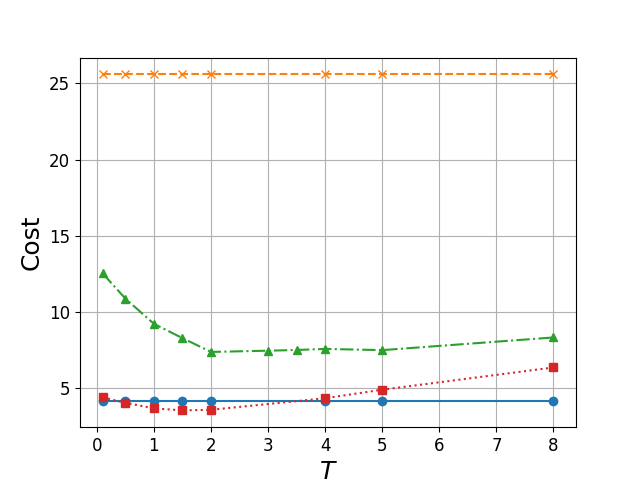}\label{fig:srv_T_weibull_uniform}} &
        \subfloat[cost vs. $\lambda$- Server]{\includegraphics[width=\matrixCellWidth]{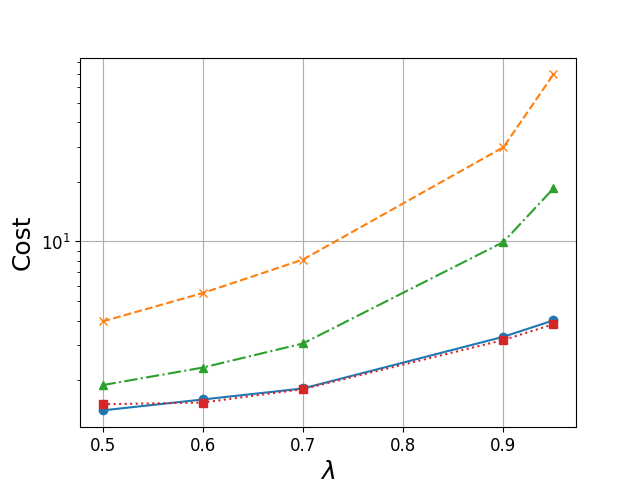}\label{fig:srv_arriving_weibull_uniform}}
        \tabularnewline
        \multicolumn{3}{c}{\subfloat{\includegraphics[width=0.5\linewidth]{figs/graphs/legend.png}}}	

    \end{tabular}
    \caption{Cost in the external cost and server cost models using the uniform predictor for Weibull service time. The default costs for the external model are $c_1= 0.5, c_2=2$ and for the server cost model are $c_1= 0.01, c_2=0.05$ (a + d) Cost vs. $c_2$ when $\lambda = 0.9$ and $T=1$ (b + e) Cost vs. $T$ when $\lambda = 0.9$ (c + f) Cost vs. $\lambda$ when $T=1$.}
    \label{fig:sim_uniP_dist_weibull}
\end{figure}

\section{\algtwo{} Proofs}
\label{appendix:delayp_proofs}

\lemmadelaypexts*

\begin{proof}
We first find the worst future rank and then calculate $X^{\text{new}}[rank_{d_J}^{\text{worst}} (a)]$, $X_{0}^{\text{old}}[r_{worst}]$ and $X_{i}^{\text{old}}[r_{worst}]$ for short jobs.
As both the external cost model and the server cost model treat short jobs the same, and their worst future rank is the same, the analysis holds for both.

For short jobs, the rank function is monotonic, therefore J's worst future rank is its initial rank:
\[ rank_{d_J}^{\text{worst}} (a)=  \langle 1, -a \rangle \] and  $r_{worst}= rank_{d_J}^{\text{worst}} (0) = \langle 1, 0 \rangle$.

$X^{\text{new}}[rank_{d_J}^{\text{worst}} (a)]$: Since short jobs have higher priorities than long jobs, and they are scheduled FCFS, a new job does not preempt $J$.  Hence $X_{x_K}^{new}[\langle 1, - a \rangle]  = 0$.

$X_{0}^{\text{old}}[r_{worst}]$:
If an old job $I$ is short, it remains original until it is completed. Otherwise, it remains original only for $L$ times. Thus, for $i \ge 1$, $X_{i, x_I}^{\text{old}}[\langle 1, 0 \rangle] = 0$ and

\[
X_{0, x_I}^{\text{old}}[\langle 1, 0 \rangle] =
\begin{cases}
    x_I   & \text{if $I$ is short} \\
    L & \text{if $I$ is long}
\end{cases}
\]

\begin{align*}
\mathbb{E}[X_{0}^{\text{old}}[\langle 1, 0 \rangle]] = & \int_{0}^{L} x f(x) dx + \int_{L}^{\infty} L f(x) dx
\end{align*}

\begin{align*}
\mathbb{E}[(X_{0}^{\text{old}}[\langle 1, 0 \rangle])^2]] = & \int_{0}^{L} x^2 f(x) dx + \int_{L}^{\infty} L^2 f(x) dx
\end{align*}

Applying Theorem 5.5 of SOAP~\cite{scully2018soap} yields the result
\begin{align*}
    \mathbb{E}[T]_{ext}^{{\algtwo{}, S}}  = &\frac{\lambda}{2(1-\rho_{L})} \Bigg(\int_{0}^{L} x^2 f(x) \, dx + \int_{L}^{\infty} L^2 f(x) dx \Bigg) + \int_{0}^{L} x f(x) dx,
\end{align*} where $\rho_{L} = \lambda \left(\int_{0}^{L} x f(x) dx + \int_{L}^{\infty} L f(x) dx\right)$.
\end{proof}

\lemmadelaypextl*

\begin{proof}

To analyze \algtwo{} for a long job in the external cost model using SOAP, we first find the worst future rank and then calculate $X^{\text{new}}[rank_{d_J}^{\text{worst}} (a)]$, $X_{0}^{\text{old}}[r_{worst}]$ and $X_{i}^{\text{old}}[r_{worst}]$ for long job.
As described in \eqref{eq:alg2_ext_rank_equation}, the rank function for long jobs is monotonic, and every job's rank is strictly decreasing with age, thus $J$'s worst future rank is its initial rank, here: $rank_{d_J}^{\text{worst}} (a)=  \langle 2, r -a \rangle $ and $r_{worst}= rank_{d_J}^{\text{worst}} (0) = \langle 2, r \rangle$.

$X^{\text{new}}[rank_{d_J}^{\text{worst}} (a)]$:
Suppose that a new job $K$ of predicted size $r_K$ arrives when $J$ has age $a_J$.
$J$'s delay due to $K$ depends on whether $K$ is short or long. If $K$ is short then it will preempt $J$, since it has a higher priority, and be scheduled till completion. If $K$ is long and has a predicted job size less than $J$'s predicted remaining process time ($r - a_J$), $K$ will preempt $J$ and proceed until completion. Otherwise, If $K$ is long and has a predicted job size more than $J$'s predicted remaining process time, it will preempt $J$ but will be scheduled only for $L$ time.

Thus

\[X_{x_K}^{new}[\langle 2, r - a \rangle]  = 
\begin{cases}
    x_K   & \text{$K$ is short } \\
    x_K \mathds{1}(r_K < r - a) + L \cdot \mathds{1}(r_K \ge r - a) & \text{$K$ is long}
\end{cases}
\]

\begin{align*}
\mathbb{E}[X^{new}[\langle 2, r - a \rangle]] = & \int_{x=0}^{L} x f(x) dx + \int_{y = 0}^{r-a} \int_{x = L}^{\infty} x \cdot g(x,y) dx dy + \int_{y = r-a}^{\infty} \int_{x = L}^{\infty} L \cdot g(x,y) dx dy\\
\end{align*}

$X_{0}^{\text{old}}[r_{worst}]$:
Whether another job $I$ is original or recycled depends on whether it is short or long, and in the case it is long, it also depends on its predicted size relative to J's prediction.
If $I$ is short, then it remains original until its completion. Alternatively, if $I$ is long, it is original until completion if $r_I \le r$, otherwise, it is original until $L$.

\[
X_{0, x_I}^{\text{old}}[\langle 2, r \rangle] = 
\begin{cases}
    x_K   & \text{$K$ is short } \\
    x_K \mathds{1}(r_K < r ) + L \cdot \mathds{1}(r_K \ge r ) & \text{$K$ is long}
\end{cases}
\]

\begin{align*}
\mathbb{E}[X_{0}^{\text{old}}[\langle 2, r \rangle]] = & \int_{x=0}^{L} x f(x) dx + \int_{y = 0}^{r} \int_{x = L}^{\infty} x \cdot g(x,y) dx dy + \int_{y = r}^{\infty} \int_{x = L}^{\infty} L \cdot g(x,y) dx dy
\end{align*}

\begin{align*}
\mathbb{E}[(X_{0}^{\text{old}}[\langle 2, r \rangle])^2]] = & \int_{x=0}^{L} x^2 f(x) dx + \int_{y = 0}^{r} \int_{x = L}^{\infty} x^2 \cdot g(x,y) dx dy + \int_{y = r}^{\infty} \int_{x = L}^{\infty} L^2 \cdot g(x,y) dx dy
\end{align*}

$X_{i}^{\text{old}}[r_{worst}]$:
If another job $I$ is long and if $r_I > r$, then $I$ starts discarded but becomes recycled when $r_I - a = r$. This starts at age $a = r_I -r$ and continues until completion, which will be $x_I -L - a_I = x_I -L - (r_I - r)$. Thus, for $i \ge 2$, $X_{i, x_I}^{\text{old}}[\langle 2, r \rangle] = 0$. Let $t = r_I$:

\[
X_{1, x_I}^{\text{old}}[\langle 2, r \rangle] =
\begin{cases}
    0   & \text{if $I$ is short} \\
    x_I -L - (t - r) & \text{if $I$ is long}
\end{cases}
\]

\begin{align*}
\mathbb{E}[X_{1}^{\text{old}}[\langle 2, r \rangle]^2] = \int_{t = r}^{\infty} \int_{x =L+ t-r}^{\infty} g(x, t)\cdot(x -L- (t - r))^2 \cdot dx dt
\end{align*}

Applying Theorem~\ref{soaptheorem} leads to the result.

\begin{align*}
    \mathbb{E}[T(x_J, r)]_{ext}^{{\algtwo{}, L}}  = &\frac{\lambda}{2(1-\rho^{ext}_{L,r})^2} \Bigg(\int_{x=0}^{L} x^2 f(x) dx \\
    &+ \int_{y = 0}^{r} \int_{x = L}^{\infty} x^2 \cdot g(x,y) dx dy + \int_{y = r}^{\infty} \int_{x = L}^{\infty} L^2 \cdot g(x,y) dx dy \\
    &+\int_{t = r}^{\infty} \int_{x =L+ t-r}^{\infty} g(x, t)\cdot(x -L- (t - r))^2 \cdot dx dt\Bigg) +  \int_{0}^{x_J} \frac{1}{1-\rho^{ext}_{L,(r-a)^+}} \, da
\end{align*}

Where $\rho^{ext}_{L, r} = \lambda \left(\int_{x=0}^{L} x f(x) dx + \int_{y = 0}^{r} \int_{x = L}^{\infty} x \cdot g(x,y) dx dy + \int_{y = r}^{\infty} \int_{x = L}^{\infty} L \cdot g(x,y) dx dy\right)$.

\end{proof}

\lemmadelaypsrvl*

\begin{proof}
In this case, $J$'s worst future rank is $rank_{d_J}^{\text{worst}} (a)=  \langle 3, r -a \rangle $ and $r_{worst}= rank_{d_J}^{\text{worst}} (0) = \langle 3, r \rangle$. Now we calculate $X^{\text{new}}[rank_{d_J}^{\text{worst}} (a)]$, $X_{0}^{\text{old}}[r_{worst}]$ and $X_{i}^{\text{old}}[r_{worst}]$ for long jobs in the server cost most.

$X^{\text{new}}[rank_{d_J}^{\text{worst}} (a)]$:
$J$'s delay due to $K$ also depends on $K$ is short or long.
If $I$ is short, then it remains it is scheduled until its completion. Alternatively, if $I$ is long its prediction is scheduled for $c_2$. In addition, if $r_I \le r$ then it is scheduled until completion, otherwise it is scheduled until $L$.

\[X_{x_K}^{new}[\langle 3, r - a_J \rangle]  = 
\begin{cases}
    x_K   & \text{$K$ is short } \\
    (c_2 + x_K)\cdot \mathds{1}(r_K < r- a_J) + (c_2 + L) \cdot \mathds{1}(r_K \ge r- a_J) & \text{$K$ is long}
\end{cases}
\]

\begin{align*}
\mathbb{E}[X^{new}[\langle 3, r - a \rangle]] = & \int_{x=0}^{L} x f(x) dx + \int_{y = 0}^{r-a} \int_{x = L}^{\infty} (x+c_2) \cdot g(x,y) dx dy \\
& + \int_{y = r-a}^{\infty} \int_{x = L}^{\infty} (L + c_2) \cdot g(x,y) dx dy\\
\end{align*}

$X_{0}^{\text{old}}[r_{worst}]$:
The analysis is similar to the new arrival job. Whether another job $I$ is original or recycled depends on whether it is short or long, and in the case it is long, it also depends on its predicted size relative to J's prediction.

\[
X_{0, x_I}^{\text{old}}[\langle 3, r \rangle] = 
\begin{cases}
    x_K   & \text{$K$ is short } \\
    (c_2 + x_K)\cdot \mathds{1}(r_K < r) + (c_2 + L) \cdot \mathds{1}(r_K \ge r) & \text{$K$ is long}
\end{cases}
\]

\begin{align*}
\mathbb{E}[X_{0}^{\text{old}}[\langle 3, r \rangle]] =& \int_{x=0}^{L} x f(x) dx + \int_{y = 0}^{r} \int_{x = L}^{\infty} (x+c_2) \cdot g(x,y) dx dy + \int_{y = r}^{\infty} \int_{x = L}^{\infty} (L + c_2) \cdot g(x,y) dx dy\\
\end{align*}

\begin{align*}
\mathbb{E}[(X_{0}^{\text{old}}[\langle 3, r \rangle])^2]] = & \int_{x=0}^{L} x^2 f(x) dx + \int_{y = 0}^{r} \int_{x = L}^{\infty} (x+c_2)^2 \cdot g(x,y) dx dy \\
&+ \int_{y = r}^{\infty} \int_{x = L}^{\infty} (L + c_2)^2 \cdot g(x,y) dx dy\\
\end{align*}

$X_{i}^{\text{old}}[r_{worst}]$:
As described before, if another job $I$ is long and if $r_I > r$, then $I$ starts discarded but becomes recycled when $r_I - a = r$. This starts at age $a = r_I -r$ and continues until completion, which will be $x_I -L - a_I = x_I -L - (r_I - r)$. Thus, for $i \ge 2$, $X_{i, x_I}^{\text{old}}[\langle 3, r \rangle] = 0$. Let $t = r_I$:

\[
X_{1, x_I}^{\text{old}}[\langle 3, r \rangle] =
\begin{cases}
    0   & \text{if $I$ is short} \\
    x_I -L - (t - r) & \text{if $I$ is long}
\end{cases}
\]

\begin{align*}
\mathbb{E}[X_{1}^{\text{old}}[\langle 3, r \rangle]^2] = \int_{t = r}^{\infty} \int_{x =L+ t-r}^{\infty} g(x, t)\cdot(x -L- (t - r))^2 \cdot dx dt
\end{align*}

Applying Theorem~\ref{soaptheorem} leads to the result.

\begin{align*}
    \mathbb{E}[T(x_J, r)]_{srv}^{{\algtwo{}, L}}  = &\frac{\lambda}{2(1-\rho^{srv}_{L,r})^2} \Bigg(\int_{x=0}^{L} x^2 f(x) dx + \int_{y = 0}^{r} \int_{x = L}^{\infty} (x+c_2)^2 \cdot g(x,y) dx dy \\
    & + \int_{y = r}^{\infty} \int_{x = L}^{\infty} (L + c_2)^2 \cdot g(x,y) dx dy \\
    &+\int_{t = r}^{\infty} \int_{x =L+ t-r}^{\infty} g(x, t)\cdot(x -L- (t - r))^2 \cdot dx dt\Bigg) +  \int_{0}^{x_J} \frac{1}{1-\rho^{srv}_{L,(r-a)^+}} \, da
\end{align*}

Where $\rho^{srv}_{L, r} = \lambda \left(\int_{x=0}^{L} x f(x) dx + \int_{y = 0}^{r} \int_{x = L}^{\infty} (x+c_2) \cdot g(x,y) dx dy + \int_{y = r}^{\infty} \int_{x = L}^{\infty} (L + c_2) \cdot g(x,y) dx dy\right)$.

\end{proof}

\end{document}